\documentclass{article}
\usepackage[affil-it]{authblk}
\usepackage[utf8]{inputenc}
\usepackage[a-1b]{pdfx}
\usepackage{amsthm}
\usepackage{hyperref}
\usepackage{booktabs}
\usepackage{adjustbox}
\newtheorem{theorem}{Theorem}

\newtheorem{corollary}[theorem]{Corollary}
\newtheorem{proposition}[theorem]{Proposition}
\usepackage[margin=1in]{geometry}

\usepackage{graphics,multicol}
\usepackage{natbib}
\usepackage{epsfig}
\usepackage{wrapfig}
\usepackage{tikz}
\usetikzlibrary{positioning,arrows.meta}
\usepackage{float}
\usepackage{subcaption}
\usepackage{amsmath,amsfonts}
\usepackage{amssymb}
\usepackage{rotating}
\usepackage{pdfpages}
\usepackage{color}
\usepackage{epic}
\usepackage{eepic}
\usepackage[english]{babel}
\usepackage{enumerate}
\usepackage{url}
\usepackage{xcolor}
\usepackage{algorithm}
\usepackage{algpseudocode}
\usepackage{tabularx}
\usepackage{array}
\usepackage{booktabs}

\newcommand{\R}{\mathbb{R}}

\newcommand{\Dcal}{\mathcal{D}}

\newcommand{\Dist}{\mathrm{Dist}}

\title{Stochastic Dynamics on Persistence Diagram Space via Reinforcement Learning}

\author[1]{Farzana Nasrin}

\affil[1]{Department of Mathematics,
       University of Tennessee,
       Knoxville, TN -37916, USA}

\date{}

\begin{document}

\maketitle

\begin{abstract}

Persistence diagrams (PDs) provide stable and interpretable summaries of multiscale
topological structure. While substantial progress has been made in the
statistical analysis of PDs, existing literature often treats diagrams
as static objects and provide limited frameworks for probabilistic
modeling and stochastic evolution on PD space.
We introduce a reinforcement learning framework for stochastic
dynamics on PD space, where diagrams evolve through
topology aware local edit operations. The dynamics define controlled Markov
processes on spaces of finite PDs with variable
cardinality. We establish conditions under which the induced Markov
chains are irreducible, aperiodic, and geometrically ergodic, implying
the existence of unique stationary probability laws on PD space.
To guide the dynamics toward scientifically relevant topological targets, we formulate objectives that encompass distribution matching, task specific topological statistics, and structure-preserving compression.
The resulting rewards balance task specific distributional targets, diagram fidelity, and complexity reduction, and yield a framework for adaptive topological simplification and probabilistic modeling.
Experiments on synthetic and neuroimaging PDs
demonstrate that the proposed framework can preserve dominant
topological structure while reducing diagram complexity. 
% These results suggest a new probabilistic
% perspective in which PDs
% are modeled as states of stochastic topological processes rather than
% static summaries alone.
\end{abstract}

\textbf{keywords.}
Persistence diagrams, topological data analysis,
reinforcement learning, stochastic dynamics,

probabilistic modeling

\section{Introduction}
\label{sec:intro}
Topological data analysis (TDA) provides tools for extracting qualitative and multiscale structure from complex data. Among these tools, persistence diagrams (PDs) encode the birth and death of homological features, such as connected components, loops, and voids, across a filtration~\citep{edelsbrunner2008persistent,oudot2015persistence}. Their stability and interpretability have made them useful in statistical learning and applications~\citep{Liu2019, Gabella2021, Heydenreich2021, Maroulas2021, Townsend2020, Ichinomiya2020, Bernstein2020, Deshmukh2023, Sun2021, Tymochko2020, Amezquita2023, Papamarkou2022randomPD, Nasrin_2024}. Nevertheless, collections of PDs obtained from noisy or heterogeneous observations present two practical challenges: characterizing the topological variability of the underlying population~\citep{mileyko2011probability,turner2014frechet,maroulas2019nonparametric} and reducing diagram complexity while controlling distortion of the information needed for subsequent analysis~\citep{sheehy2021sketching}.

Existing methods analyze PDs through functional summaries, kernels, probability measures, and point-process models~\citep{bubenik2015landscapes,reininghaus2015stable,carriere2017sliced,mileyko2011probability,Adler2019,maroulas2019bayesian}. Our objective is complementary: to learn transformations that address two related problems, namely, reproducing population-level topological variability and simplifying individual diagrams while controlling information loss. The first supports applications such as uncertainty analysis and rare event simulation, whereas the second can reduce the cost of repeated diagram comparisons while preserving subject level topology and cohort organization. Both problems are naturally sequential because each addition, deletion, or movement changes the current diagram and therefore affects which subsequent edits are available and useful. A fixed pointwise rule cannot account for these cumulative effects or adapt its decisions to the remaining transformation budget. We therefore formulate PD transformation as a sequential decision problem in which a reinforcement learning (RL) policy observes the current diagram, selects a valid topology aware edit, and receives task specific feedback based on the consequences of the resulting sequence~\citep{sutton2018reinforcement}. In this formulation, RL provides a common mechanism for learning both stochastic population behavior and adaptive topology preserving simplification.

For a fixed policy, topology aware edit operations induce a Markov transition kernel on PD space, whose stationary law provides a probabilistic model for the policy's long run behavior. Using general-state-space Markov chain arguments based on irreducibility, aperiodicity, and drift and minorization conditions~\citep{meyn2012markov}, we establish geometric ergodicity and the existence of a unique stationary distribution, with convergence independent of the initial diagram. The ergodic theorem then justifies approximating expectations under the stationary distribution by averages along finite policy generated trajectories and provides the basis for consistency of the empirical distributional objective. Stable discrepancies, including sliced Wasserstein constructions and kernel maximum mean discrepancy (MMD)~\citep{carriere2017sliced,gretton2012kernel}, can therefore compare the policy induced behavior with empirical target distributions or task specific topological summaries. RL controls this behavior by selecting successive edits according to a task specific reward, providing a common framework for population level modeling, rare event recovery, uncertainty oriented sampling, and topology aware simplification.

We evaluate the framework through synthetic rare event recovery and topology preserving compression of ABIDE neuroimaging PDs~\citep{dimartino2014abide}, illustrating its use for both distributional modeling and sequential simplification.
The main contributions of this paper are summarized as follows.

\begin{itemize}

\item We introduce an RL framework for stochastic dynamics on PD space, in which topology aware edit operations induce controlled Markov processes acting directly on PD points.

\item We establish irreducibility, aperiodicity, and geometric ergodicity of the induced dynamics, yielding unique stationary probability laws and convergence of the policy-controlled process.

\item We develop distributional objectives based on stable statistical discrepancies, enabling probabilistic modeling and topology preserving simplification directly on PD space.

\item Through synthetic and ABIDE experiments, we demonstrate rare event recovery and sequential compression that preserves both diagram level topology and cohort level geometry.

\end{itemize}

The remainder of the paper is organized as follows. Section \ref{sec:background} reviews PDs, statistical models on PD space, and the required background on RL and MCs. Section~ \ref{sec:framework} presents the proposed framework, including the diagram state space, topology aware edit operations and Section \ref{sec:theory} establishes the theoretical properties, including existence of a unique stationary distribution and convergence of empirical trajectory measures. Section \ref{sec:algorithms} discusses the practical realization of the framework. Section~\ref{sec:experiments} presents the synthetic rare event recovery and ABIDE neuroimaging experiments. Finally, Section \ref{sec:discussion} discusses the implications, limitations, and future directions of the work, while some technical proofs are provided in the Appendix.

%%%%%%%%%%%%%%%%%%%%%%%%%%%%%%%%%%%%%%%%%%%%%%%%
\section{Background} \label{sec:background}

Persistence diagrams (PDs) provide multiscale summaries of topological features in data, making them well suited for probabilistic modeling and statistical learning. In this section, we review the construction of PDs, existing statistical approaches on diagram space, and the reinforcement learning (RL) background needed for our framework.

\subsection{Persistence Diagrams}

We briefly recall the construction of PDs; see \cite{edelsbrunner2008persistent,oudot2015persistence,ghrist2008barcodes} for details. 
Let $(X,d)$ be a metric space or a domain endowed with a filtration. 
Two standard examples are:
\begin{enumerate}
    \item \textbf{Sublevel set filtrations:} In a \emph{sublevel set filtration} \citep{edelsbrunner2008persistent,oudot2015persistence}, one considers a scalar function $f:X \to \mathbb{R}$, such as pixel intensity in a grayscale image. The sublevel sets $S_r = \{x \in X: f(x) \leq r\}$ expand as the threshold $r$ increases, and topological features such as connected components ($H_0$), loops ($H_1$), and voids ($H_2$) appear and disappear. Their birth and death times define points in a PD, where persistence reflects lifetime across scales. Figure~\ref{fig:SL} top row illustrates this process: as $r$ increases, new components form, merge, and occasionally create loops, producing a PD that summarizes image topology. 
    \begin{figure}[h!]
    	\centering
    	\subfloat[]{\includegraphics[width=1.25in, height = 
    	1in]{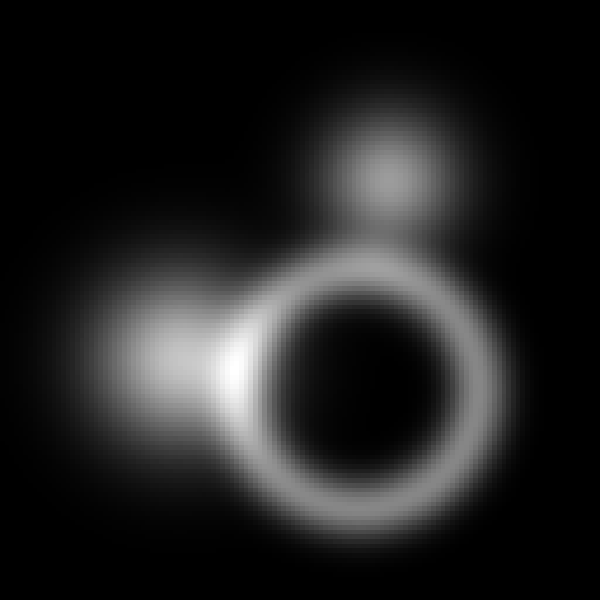}}\hspace{0.1in}
            	\subfloat[]{\includegraphics[width=1.25in, height = 
    	1in]{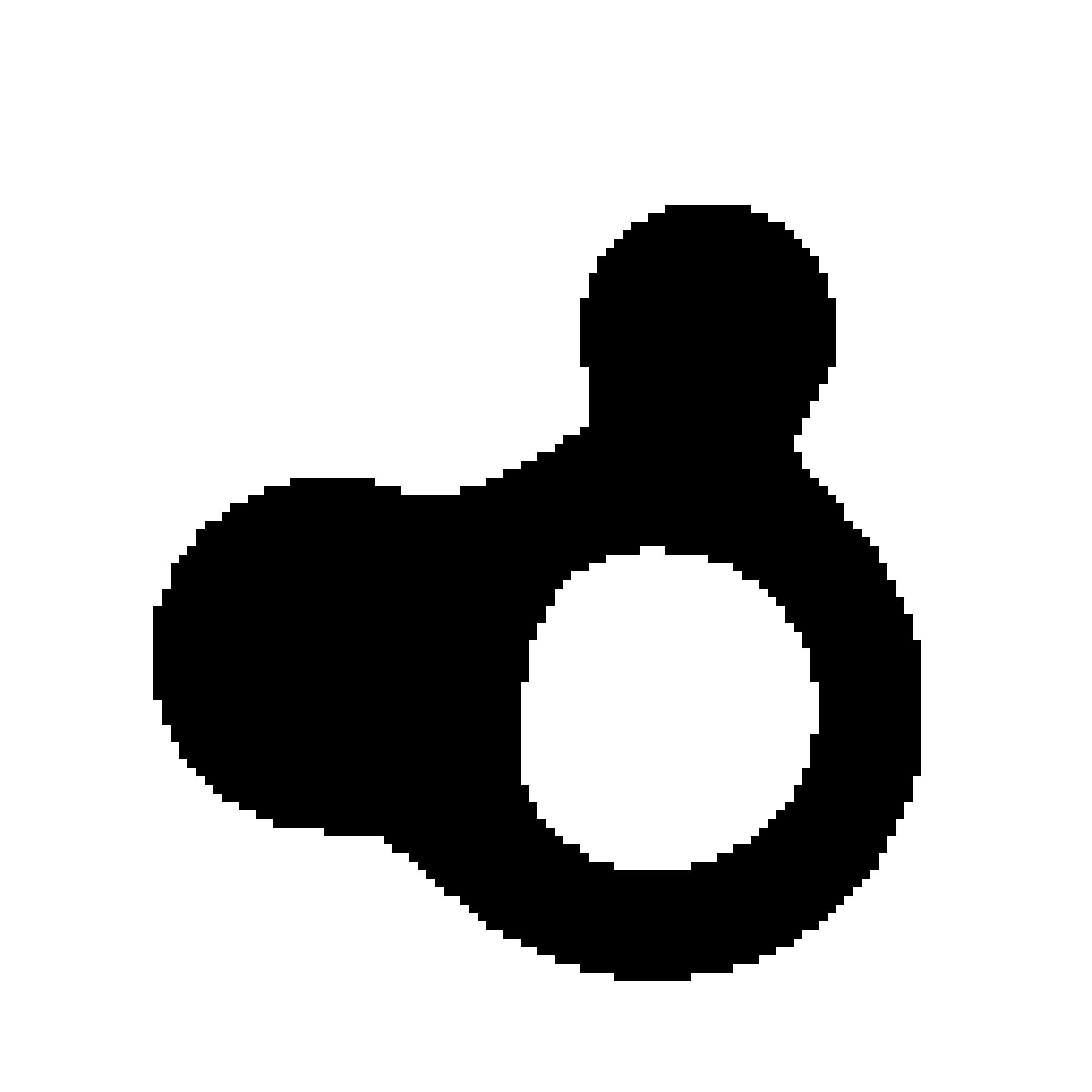}}\hspace{0.1in}
                \subfloat[]{\includegraphics[width=1.25in, height = 
    	1in]{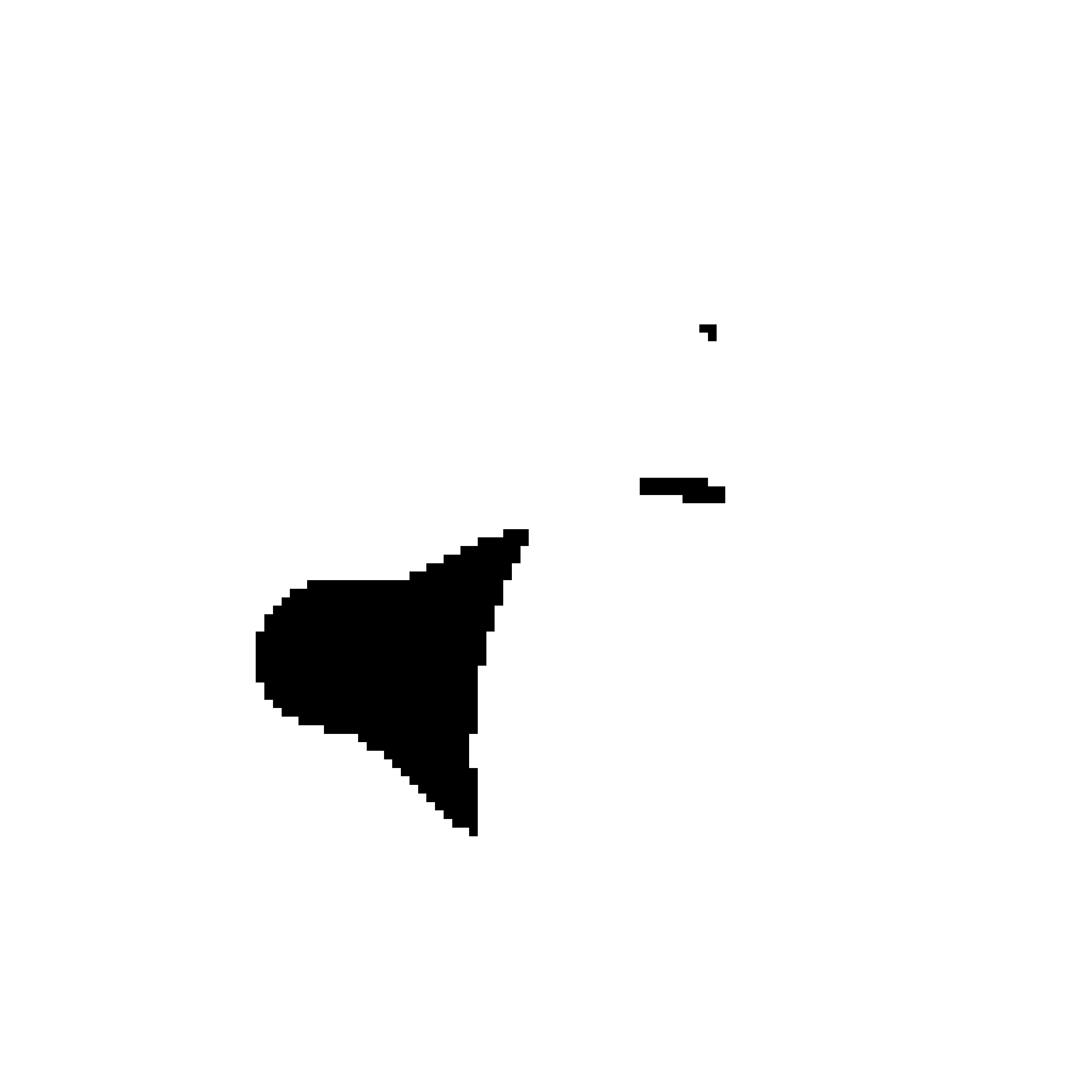}} \hspace{0.1in}
    		\subfloat[]{\includegraphics[width=1.5in, height = 
    	1.25in]{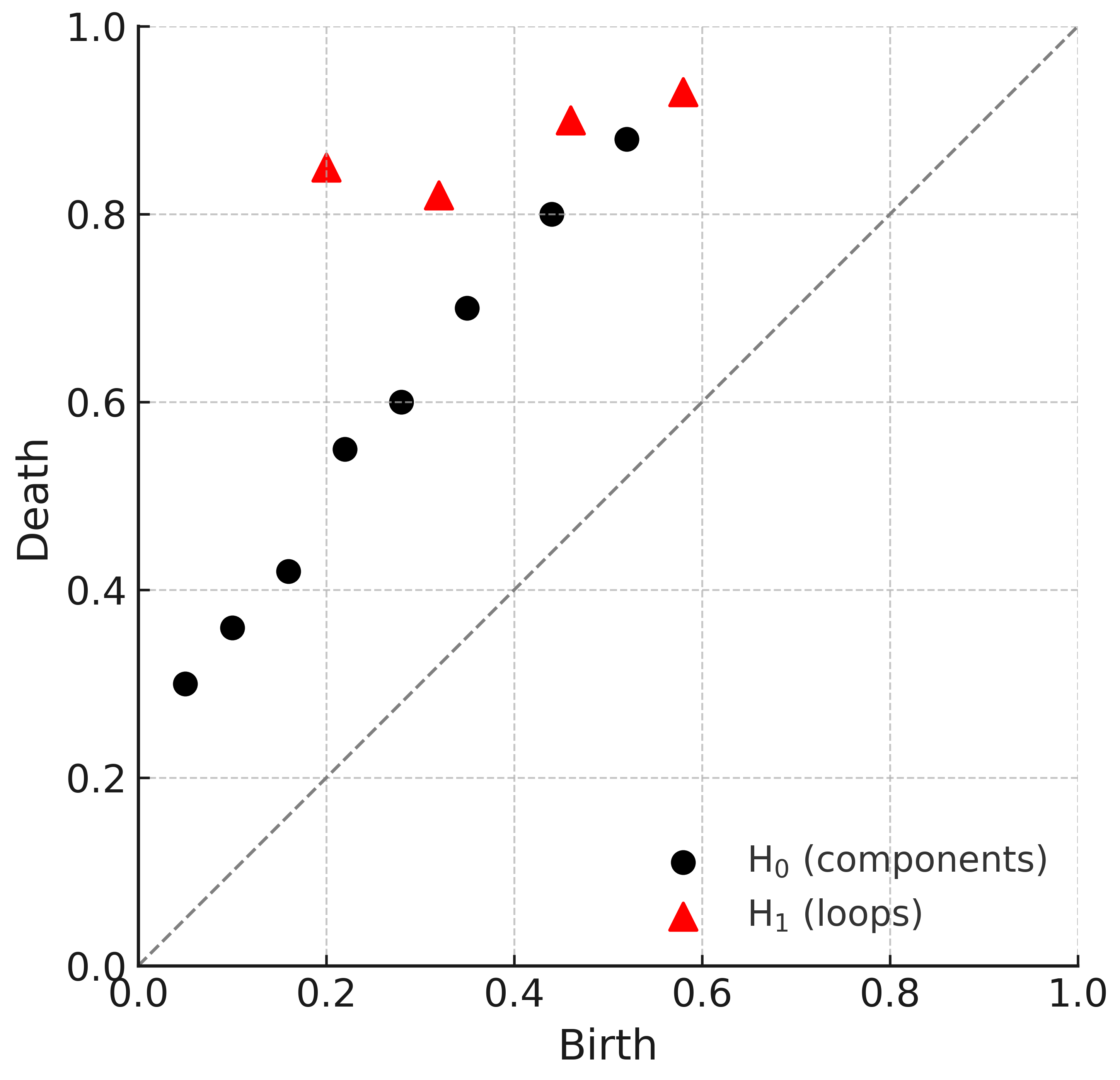}}\\
    	
\subfloat[]{\includegraphics[width=1.25in, height = 
    	1in]{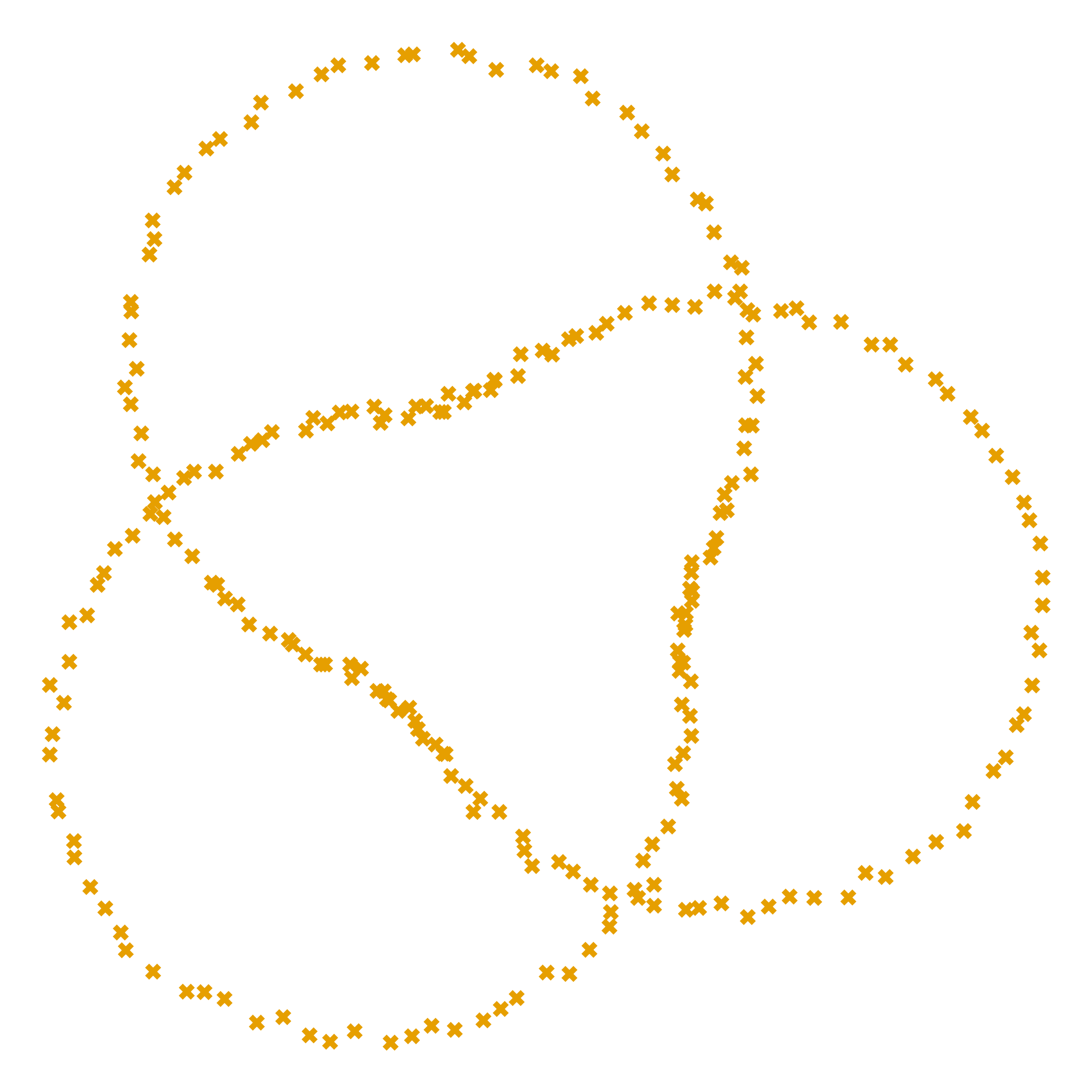}}\hspace{0.1in}
    	\subfloat[]{\includegraphics[width=1.25in, height = 
    	1in]{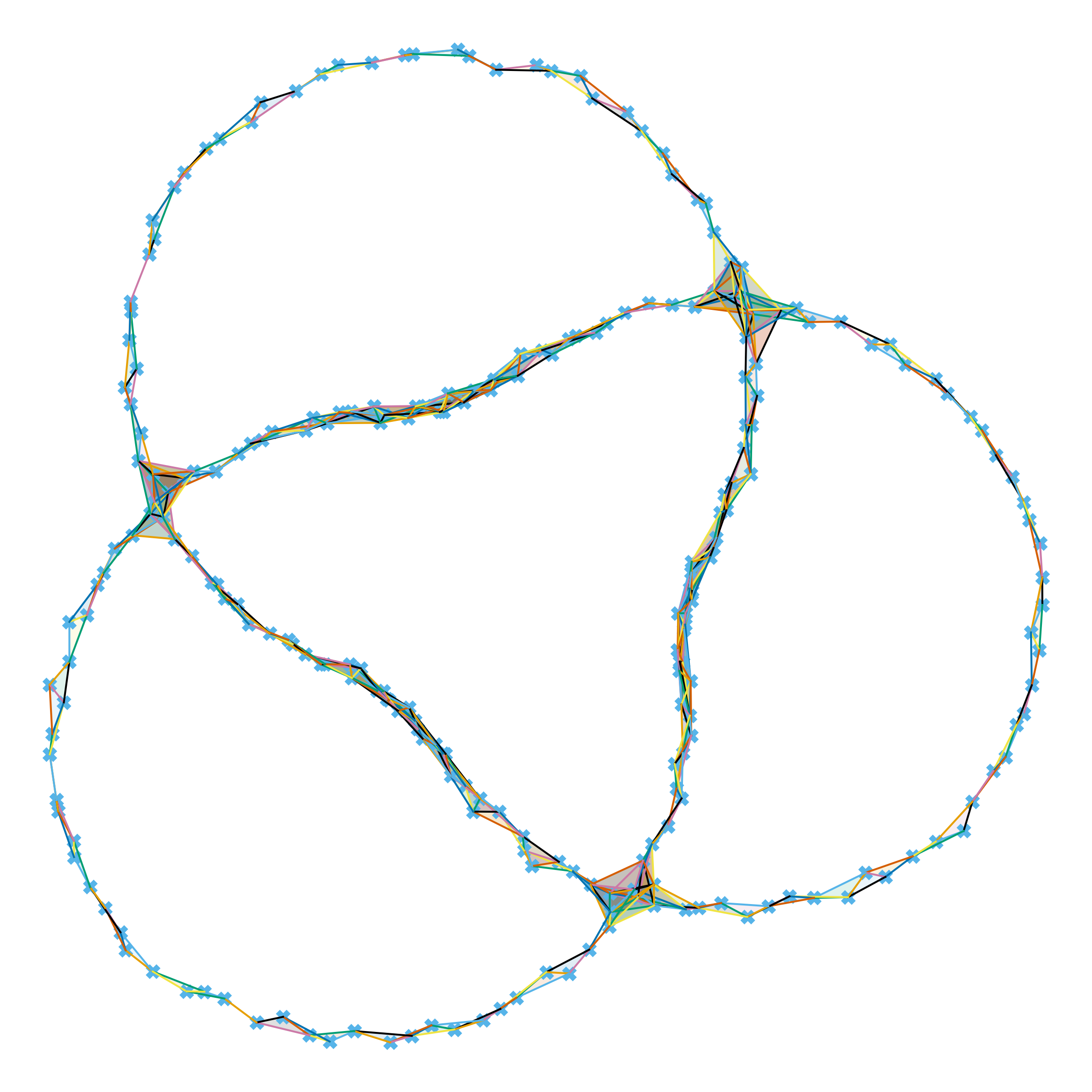}}\hspace{0.1in}
    	\subfloat[]{\includegraphics[width=1.25in, height = 
    	1in]{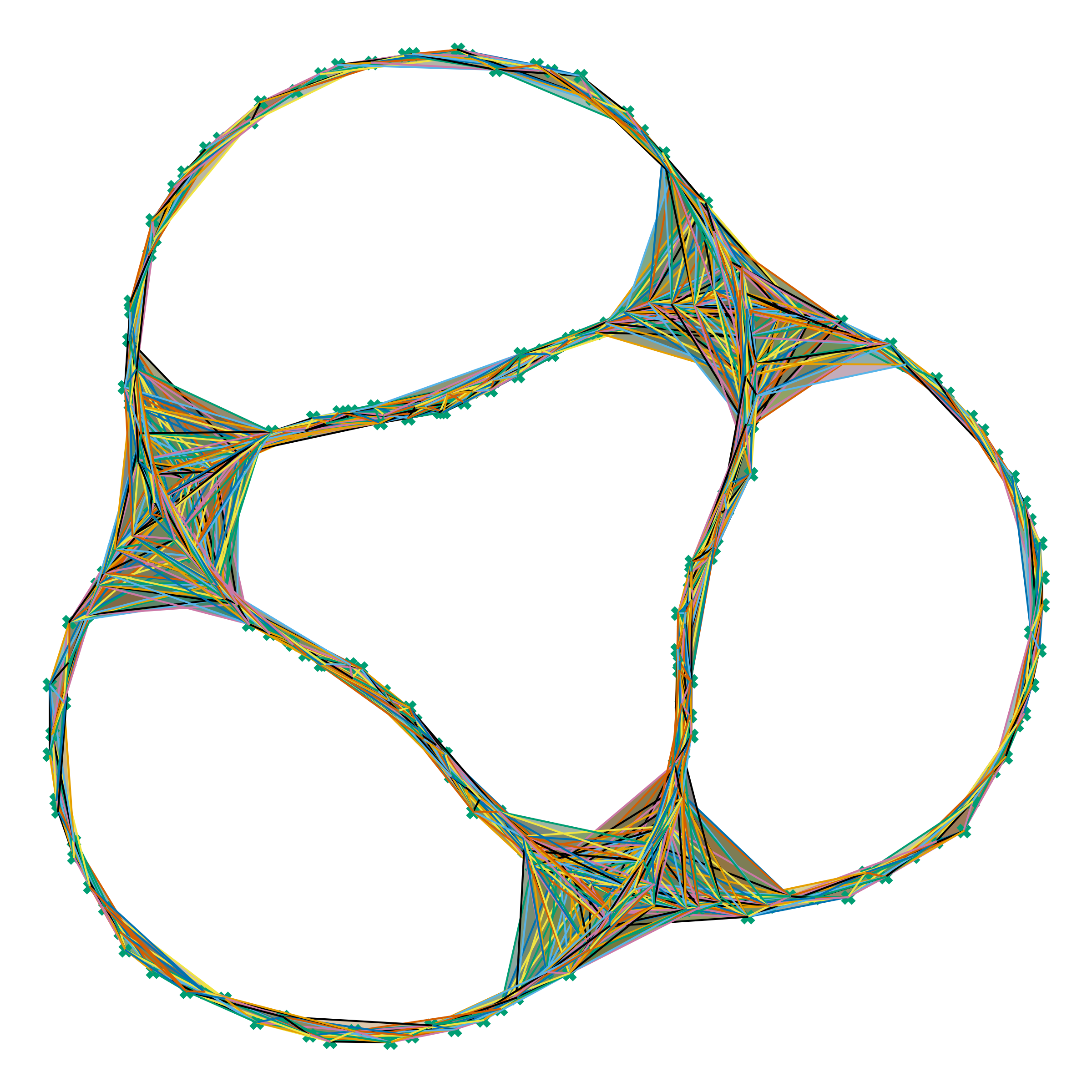}} \hspace{0.1in}
        \subfloat[]{\includegraphics[width=1.5in, height = 
    	1.25in]{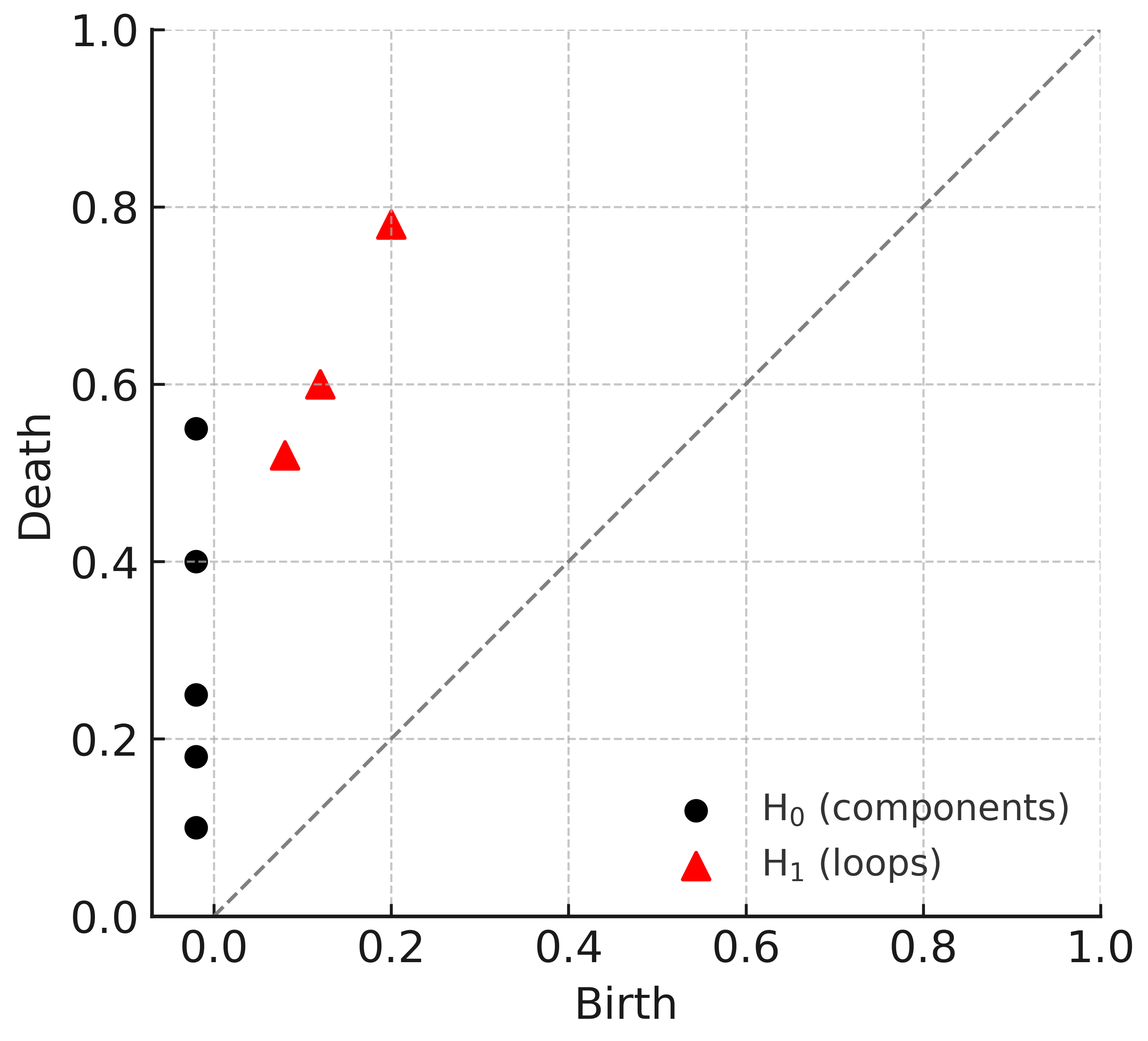}}\\

    	\caption{{Top row: (a) is a toy grayscale image, (b) - (c) are two intermediate steps of the sublevel set filtration corresponding to $r=0.2$, and $r=0.6$, respectively,  and (d) is the corresponding PD.   
        Bottom row: VR filtration for a trefoil-knot point cloud. 
 A 2D projection of a noisy trefoil sampling is shown (e). 
As the distance threshold $\epsilon$ increases, edges are added to form Rips complexes (f) - (g), and topological features appear and vanish across scales.   Their birth and death times are summarized in the PD (h).
    	}}
    	\label{fig:SL}
    \end{figure}
    \item \textbf{Vietoris Rips filtrations:} The \emph{Vietoris Rips (VR) complex} \citep{ghrist2008barcodes} provides a natural construction for point cloud data. Here, points are connected if their pairwise distance is below a threshold $\epsilon$, and higher dimensional simplices are added whenever all faces are present. As $\epsilon$ grows, connected components merge, loops emerge, and higher order cycles eventually fill in. Because it relies only on pairwise distances, the Rips filtration applies broadly to arbitrary metric datasets. Figure~\ref{fig:SL}  bottom row demonstrates this using a trefoil-knot point cloud: increasing $\epsilon$ introduces edges and simplices, recording topological events in a PD that captures the geometry of the knot. 
    % This example highlights how PDs provide a unified language for extracting multiscale structure from both scalar fields and metric data.

%     \begin{figure}[!h]
%     	\centering
%     	\subfloat[]{\includegraphics[width=1.5in, height = 
%     	1.5in]{figures/trefoil_point_cloud.png}}\hspace{0.1in}
%     		\subfloat[]{\includegraphics[width=2in, height = 
%     	1.5in]{figures/trefoil_rips_pd_shifted.png}}\\
    	
%     	\subfloat[]{\includegraphics[width=1.5in, height = 
%     	1.5in]{figures/trefoil_rips_eps_006.png}}\hspace{0.1in}
%     	\subfloat[]{\includegraphics[width=1.5in, height = 
%     	1.5in]{figures/trefoil_rips_eps_01.png}}\hspace{0.1in}
%     	\subfloat[]{\includegraphics[width=1.5in, height = 
%     	1.5in]{figures/trefoil_rips_eps_016.png}}

%     	\caption{VR filtration for a trefoil-knot point cloud. 
% A 2D projection of a noisy trefoil sampling is shown (a). 
% As the distance threshold $\epsilon$ increases, edges are added to form Rips complexes (c) - (e), 
% and topological features appear and vanish across scales. 
% Their birth and death times are summarized in the PD (b), with black circles for $H_0$ (components) and red triangles for $H_1$ (loops).}
%     	\label{fig:rips}
%     \end{figure}

\end{enumerate}
% As the scale parameter (e.g., $r$ or $\varepsilon$) increases, homological features appear and disappear: connected components ($H_0$), loops ($H_1$), and higher-dimensional cycles. 
Each feature is assigned a birth time $b$ (when it appears) and death time $d$ (when it vanishes), producing a multiset of points $(b,d)$ in the wedge
\[
\mathbb{W} = \{(b,d)\in\R^2: b<d\}.
\]
The resulting multiset of birth--death pairs is called the PD.
Stability theorems show that small perturbations of the input data induce
bounded changes in PDs under bottleneck and Wasserstein
distances~\citep{cohen2007stability,chazal2015stochastic}. We denote by
$\mathcal{D}$ the space of all finite PDs, modeled as
finite multisets of points in the wedge $\mathbb{W}$. 
This space is commonly equipped with bottleneck or $p$-Wasserstein metrics,
which measure optimal matchings between diagram points while allowing
unmatched features to be paired with the diagonal.  This metric structure on $\mathcal{D}$ provides the geometric foundation for defining probability measures and stochastic dynamics directly on PD space.

\subsection{Statistical Models for Persistence Diagrams}

A wide range of statistical methodologies have been developed for PDs:

\paragraph{Functional and vector representations.}
Persistence landscapes and silhouettes embed PDs into functional spaces, enabling classical statistical analysis and convergence theory~\citep{bubenik2015landscapes,chazal2015stochastic}. Persistence images instead provide stable finite-dimensional vectors suitable for standard statistical and machine learning methods~\citep{adams2017persistence}.

\paragraph{Kernel, optimal-transport, and geometric methods.}
Positive definite and persistence weighted kernels provide reproducing kernel representations of PDs~\citep{reininghaus2015stable,kwitt2015statistical,kusano2016persistence}, while sliced Wasserstein constructions enable stable and computationally tractable diagram comparisons~\citep{carriere2017sliced}. The metric geometry of PD space also supports probability measures, Fr\'echet summaries, and nonparametric density estimation~\citep{mileyko2011probability,turner2014frechet,maroulas2019nonparametric}.

% \paragraph{Geometric probability and nonparametric inference.}The metric geometry of PD space supports probability measures and statistical quantities such as expectations, variances, and Fréchet means~\citep{mileyko2011probability,turner2014frechet}. Nonparametric methods have also been developed to estimate probability density functions for random PDs using random set and kernel density formulations~\citep{maroulas2019nonparametric}.

\paragraph{Point-process, Bayesian, and neural models.}
Point-process models treat off-diagonal persistence points as structured spatial patterns. Gibbs, Poisson, and pairwise-interaction formulations capture dependence among points, support posterior intensity estimation, and generate random PDs~\citep{Adler2019,maroulas2019bayesian,Papamarkou2022randomPD}, while Bayesian extensions permit inference on generator parameters~\citep{Nasrin_2024}. Neural methods instead learn task-dependent representations, including trainable persistence-barcode features and the permutation-invariant PersLay architecture~\citep{hofer2019learning,Carrire2019PersLayAN}.

% \paragraph{Neural representation learning.}Neural approaches learn task dependent representations of PDs rather than relying on a fixed vectorization. Examples include trainable representations of persistence barcodes~\citep{hofer2019learning} and PersLay, a permutation-invariant neural network layer designed specifically for PDs~\citep{carriere2020perslay}.

These approaches provide powerful tools for representing, comparing, estimating, and generating PDs. Our objective is complementary: we model a diagram as the state of a controlled stochastic process and learn how it should change through a sequence of local edit operations. This sequential formulation makes the choice of each edit depend on the current diagram and on the cumulative objective of the task, motivating the probabilistic foundation developed in the following section.

% Despite their success, most existing approaches either impose strong modeling assumptions, operate through auxiliary embedding spaces, or treat persistence diagrams as static random objects. In particular, they do not provide a mechanism for modeling sequential local transitions directly on diagram space, which is central to the stochastic framework developed in this paper.
% These limitations motivate a complementary approach: modeling PDs through stochastic transition mechanisms defined by controlled local edits.
% This requires a careful probabilistic foundation for distributions on diagram space.

% ------------------------------------------------------------
\subsection{Probability Measures on Persistence Diagrams}
\label{sec:prob_PD}

To formally define the probabilistic structure required to define and analyze distributions on PD space we view
$(\Dcal,d_{\Dcal})$ as a metric space equipped with the
$p$--Wasserstein distance
\[
W_p(D,D') 
= \left(
\inf_{\gamma\in\Gamma(D,D')}
\sum_{(x,y)\in\gamma} \|x-y\|_\infty^p
\right)^{1/p},
\]
where $\Gamma(D,D')$ denotes bijections between diagrams after including
diagonal points with infinite multiplicity.
The bottleneck distance $W_\infty$ is the $p\to\infty$ case.
These metrics satisfy fundamental stability properties under perturbations of the
underlying data~\citep{cohen2007stability,chazal2015stochastic}.

    Persistence diagrams equipped with bottleneck or $p$-Wasserstein distances form a metric space suitable for probability theory, weak convergence, and statistical analysis \citep{mileyko2011probability,turner2014frechet,chazal2015stochastic}. Let $\mathcal B(\mathcal D)$ denote the Borel $\sigma$-algebra generated by the open subsets of $(\mathcal D,d_{\mathcal D})$, and let $\mathcal P(\mathcal D)$ denote the set of all Borel probability measures on $(\mathcal D,\mathcal B(\mathcal D))$. A probability measure $\mu\in\mathcal P(\mathcal D)$ therefore represents a distribution on PD space, arising naturally from stochastic filtrations, sampling variability, Bayesian models, or algorithmically generated stochastic dynamics.

% Given observed diagrams $\{D_i\}_{i=1}^n$, the empirical distribution is
% \[
% \hat{\mu}
% =
% \frac1n\sum_{i=1}^n \delta_{D_i}.
% \]

Let $\{\mu_n\}_{n\geq1}\subset \mathcal P(\mathcal D)$ be a sequence of
probability measures on PD space.
We say that $\mu_n$ converges weakly to $\mu$, written
$\mu_n \Rightarrow \mu$, if
\[
\int f\, d\mu_n \;\to\; \int f\, d\mu
\qquad\text{for all bounded continuous } f:\mathcal D\to\mathbb R.
\]
% Intuitively, weak convergence means that bounded continuous summary
% statistics of the generated diagrams converge in expectation, which is
% an appropriate notion of convergence for variable cardinality,
% non-Euclidean objects such as PDs. This notion will be
% used when discussing convergence of the law of an RL driven Markov chain
% to its stationary distribution.
% ------------------------------------------------------------
\subsection{Reinforcement Learning and Markov Chains}
\label{sec:rl_background}

Reinforcement learning concerns sequential decision making in stochastic environments ~\citep{sutton2018reinforcement,haarnoja2018soft}. Let $\mathcal S$ denote the state space and $\mathcal A$ the action space. We use $S_t$ and $A_t$ for the random state and action at time $t$, and $s_t$ and $a_t$ for their realized values. More generally, $s\in\mathcal S$ denotes a current state, $a\in\mathcal A$ an action, and $s'\in\mathcal S$ a possible next state.
For the purposes of this paper, a Markov decision process is specified by:
\begin{itemize}
    \item a state space $\mathcal S$ with Borel $\sigma$-algebra $\mathcal B(\mathcal S)$;
    \item an action space $\mathcal A$ with Borel $\sigma$-algebra $\mathcal B(\mathcal A)$;
    \item a transition kernel $P$, where $P(E\mid s,a)=\Pr(S_{t+1}\in E\mid S_t=s,A_t=a)$, $E\in\mathcal B(\mathcal S)$; and
    \item a one-step reward function $r:\mathcal S\times\mathcal A\times\mathcal S\to\mathbb R$, where $r(s,a,s')$ evaluates the transition from $s$ to $s'$ under
action $a$.
\end{itemize}

\noindent A stochastic policy $\pi_\theta$ assigns a probability distribution over
actions to each state, where $\theta\in\Theta$ is the vector of trainable
policy parameters. For $U\in\mathcal B(\mathcal A)$,
\[
\pi_\theta(U\mid s)
=
\Pr(A_t\in U\mid S_t=s).
\]
When $\mathcal A$ is finite, $\pi_\theta(a\mid s)$ denotes the probability
of selecting action $a$ in state $s$.
Given the current state, the agent and environment generate
$
A_t\sim\pi_\theta(\cdot\mid S_t)$, and 
$S_{t+1}\sim P(\cdot\mid S_t,A_t)$.
For a fixed policy, $\{S_t\}_{t\geq0}$ is a Markov chain with transition
kernel
\[
P_\theta(E\mid s)
=
\int_{\mathcal A}
P(E\mid s,a)\,d\pi_\theta(a\mid s),
\qquad E\in\mathcal B(\mathcal S).
\]
When $\mathcal A$ is finite, this becomes
$
P_\theta(E\mid s)
=
\sum_{a\in\mathcal A}
\pi_\theta(a\mid s)P(E\mid s,a).
$
A probability measure $\mu_\theta\in\mathcal P(\mathcal S)$ is stationary
for the policy induced chain if
\[
\mu_\theta(E)
=
\int_{\mathcal S}
P_\theta(E\mid s)\,d\mu_\theta(s),
\qquad E\in\mathcal B(\mathcal S).
\]
% In conventional reinforcement learning, policies are typically optimized to maximize long-term cumulative reward. In contrast, our primary object of interest is the stationary distribution $\mu_\theta$ induced by the policy. By designing reward functions that reflect statistical or geometric objectives, the policy can be optimized so that its long-run behavior approximates a desired distribution on PDs. This connection between policy optimization and stationary laws forms the foundation of the framework developed in Section~\ref{sec:framework}.

%%%%%%%%%%%%%%%%%%%%%%%%%%%%%%%%%%%%%%%%%%%%%%%%

\section{Policy Controlled Dynamics on PD Space}
\label{sec:framework}

In this section, we formulate PDs as states of a
policy controlled stochastic process. The goal is to define valid
transition dynamics directly on diagram space and to use
RL to shape the long run distribution of the induced
Markov chain (MC). 
% We first specify the compact diagram state space and the
% local edit operations, then introduce the policy dependent transition
% kernel and the distributional objective used to align the resulting
% stationary law with observed PDs.

\subsection{PD State Space and Edit Operations}
 \label{sec:state_action}

Recall from Section \ref{sec:background} that PD points are defined on the wedge
$\mathbb W
=
\{(b,d)\in\mathbb R^2:0\leq b<d\}
$. Let $K\subset\mathbb W$ be a compact
set of admissible PD coordinates. Let $\mathcal D_K$ denote
the space of all finite PDs whose off-diagonal points
lie in $K$.
We treat $\Dcal_K$ as our state space. 
Restricting the dynamics to a compact wedge simplifies the
ergodicity analysis developed later in the paper by ensuring bounded
PD coordinates and controlled diagram complexity.
The stochastic dynamics are generated through a collection of local edit
operations on PDs. 
% While the basic operations are generic,
% their proposal distributions may be parameterized in topology aware ways
% using persistence values or empirical geometric structure derived from data.
\begin{enumerate}
    \item \textbf{Add:} insert a new point $(b,d)\in\mathbb{W}_K$ into $D$;
    \item \textbf{Move:} perturb an existing point $(b,d)\in D$ by a small vector $(\Delta b,\Delta d)$ and project back into $\mathbb{W}_K$;
    \item \textbf{Delete:} remove an existing point from $D$.
\end{enumerate}

In practice, the proposal mechanisms for these edit operations can be made topology aware and we explain this in Section \ref{sec:algorithms}. 
% For example, add moves may sample new points from empirical distributions in birth-death coordinates, including two regime proposals that distinguish low persistence noise from high persistence signal. Similarly, delete moves may favor removal of low persistence features, while move proposals may use anisotropic perturbations to better respect the geometry of observed diagram distributions. Such data-informed proposals improve the realism of stochastic diagram evolution by preserving dominant topological structure while suppressing unstable noise.
Perturbed coordinates are projected back into the admissible region through
$
\Pi:\mathbb R^2\longrightarrow K,
$
which preserves the birth--death ordering and ensures that the resulting diagram remains in \(\mathcal D_K\); the theory requires only this validity property, not a particular projection formula. These edits resemble local mechanisms used in spatial point-process samplers and random PD generators~\citep{moller2003statistical,Papamarkou2022randomPD}, but their use is controlled by a learnable policy that adapts the dynamics to the task objective.

%     \subsubsection{Reward Function}

%     \subsubsection{Policy and Transition Kernel}

% \subsection{Distributional Learning Objective}

% \subsection{Policy-Induced Markov Dynamics}

\subsection{Policy Induced Markov Dynamics}
\label{sec:policy_mc}

We now specialize the RL framework from
Section~\ref{sec:rl_background} to PD space.
Given a current diagram $D_t \in \mathcal D_K$, the admissible
local operations in Section~\ref{sec:state_action} induce stochastic dynamics on $\mathcal D_K$, yielding a
policy controlled MC whose long run behavior defines an
implicit distributional model for PD.

\subsubsection{Policy and Transition Kernel}

% Let $\pi_\theta(a\mid D)$, with 
% $a\in\mathcal A, D\in\mathcal D_K$,
% denote a stochastic policy parameterized by $\theta\in\Theta$.
At time $t$, given the current diagram $D_t$, an action
$A_t \sim \pi_\theta(\cdot\mid D_t)$
is sampled, and the environment applies the corresponding edit
operation through a transition kernel
$
P(D_{t+1}\mid D_t,A_t).
$
Thus, the next state depends only on the current diagram and selected
action, so the resulting stochastic process
$D_{t\ge0}$ forms an MC on $\mathcal D_K$.
% Marginalizing over actions yields the policy induced transition kernel 
% \[P_\theta(D,B)=
% \sum_{a\in\mathcal A}
% \pi_\theta(a\mid D),
% P(D,B\mid a),
% \qquad
% B\in\mathcal B(\mathcal D_K),\] 
% which fully characterizes the stochastic evolution of persistence
% diagrams under policy $\pi_\theta$.
% Under suitable regularity conditions, analyzed in Section~4, the
% induced Markov chain admits a unique stationary distribution
% $\mu_\theta \in \mathcal P(\mathcal D_K)$.
For a measurable set $E\in\mathcal B(\mathcal D_K)$, define the transition kernel
$
P(E\mid D,a)
=
\Pr(D_{t+1}\in E\mid D_t=D,A_t=a).
$
% Thus, $P(E\mid D,a)$ is the probability that the next PD belongs to $E$, given the current diagram $D$ and action $a$.
For a fixed policy $\pi_\theta$, the policy induced transition kernel is
\[
P_\theta(E\mid D)
=
\sum_{a\in\mathcal A}
\pi_\theta(a\mid D)\,
P(E\mid D,a),
\qquad
E\in\mathcal B(\mathcal D_K).
\]
% A probability measure
% $\mu_\theta\in\mathcal P(\mathcal D_K)$ is stationary for
% $P_\theta$ if
% \[
% \mu_\theta(E)
% =
% \int_{\mathcal D_K}
% P_\theta(E\mid D)\,d\mu_\theta(D),
% \qquad
% E\in\mathcal B(\mathcal D_K).
% \]
Under the conditions established in Section~\ref{sec:theory},
$P_\theta$ admits a unique stationary distribution, which we denote by
$\mu_\theta$. The subscript $\theta$ indicates that this distribution
depends on the policy parameters. The measure $\mu_\theta$ represents
the long run distribution of PDs generated by the
policy and is the probabilistic object that the learning framework seeks
to control.

% \subsubsection{Reward Function}

% A reward function assigns a numerical score to persistence diagrams in
% order to guide the stochastic dynamics toward desired geometric or
% topological structure.
% Let $d_{\mathcal D}$ denote a stable metric on persistence-diagram
% space, such as the bottleneck distance or a $p$-Wasserstein distance.
% A simple example of a state-based reward is
% \[
% r(D)=
% -\frac{1}{n}
% \sum_{i=1}^{n}
% d_{\mathcal D}(D,D_i),
% \]
% where $\{D_i\}_{i=1}^{n}\subset\mathcal D_K$ is a collection of target
% diagrams. The negative sign ensures that diagrams closer to the target collection
% receive larger rewards, so maximizing cumulative reward is equivalent to
% minimizing the average discrepancy from the observed persistence diagrams.
% More generally, reward functions may incorporate persistence-aware
% regularization, distribution-matching objectives, or structural
% penalties that suppress unstable near-diagonal topology while preserving
% dominant persistent features.

\subsection{Distributional Objective}
\label{sec:learning_objective}

Let $
\hat{\mu}
=
\frac{1}{n}
\sum_{i=1}^{n}
\delta_{D_i}
$
denote the empirical distribution of observed PDs
\(\{D_i\}_{i=1}^{n}\subset\mathcal D_K\).  The objective is to learn
policy such that the stationary distribution
\(\mu_\theta\) induced by the policy controlled MC reproduces
the topological characteristics of \(\hat{\mu}\) that are relevant to
the task.

To accommodate both full distributional matching and task specific objectives, let $(\mathcal Z,d_{\mathcal Z})$ be a complete separable metric space
containing the persistence aware representations, and let $\phi:\mathcal D_K\longrightarrow\mathcal Z$ be a measurable map.
%The space $\mathcal Z$ is the space in which the chosen representation of a PD takes its values. 
%For example, $\mathcal Z$ may be $\mathcal D_K$ itself, a Euclidean space of summary statistics, a function space containing Betti curves, or the binary space ${0,1}$ for a rare event indicator. 
We define
\begin{equation}
\label{eq:Jtheta}
J(\theta)
=
\Dist\!\left(
\phi_{\#}\mu_\theta,
\phi_{\#}\hat{\mu}
\right)
+
\lambda\Omega(\pi_\theta),
\end{equation}
where \(\phi_{\#}\mu\) denotes the pushforward of \(\mu\) through
\(\phi\), \(\Dist\) is a discrepancy between probability measures on
\(\mathcal Z\) and \(\Omega(\pi_\theta)\) is an optional regularization
term with $\lambda\geq0$.  
% The pushforward allows the learning objective to compare the distributions of task relevant topological quantities without requiring full diagram matching.  
Depending on the application, \(\Dist\) may be a Wasserstein or
sliced Wasserstein distance, a kernel maximum mean discrepancy, or
another discrepancy compatible with the chosen representation.
Regularization may be used to control diagram complexity, edit cost,
policy entropy, or other task dependent requirements.  The framework is
therefore not tied to one particular discrepancy or topological
summary.

Because the stationary distribution \(\mu_\theta\) is generally not
available in closed form, it is approximated using a policy-generated
trajectory \(D_1,\ldots,D_T\):
$
\hat{\mu}_{\theta,T}
=
\frac{1}{T}
\sum_{t=1}^{T}
\delta_{D_t}.
$
The corresponding empirical objective is
\begin{equation}
\label{eq:empirical_distributional_objective}
\hat J_T(\theta)
=
\Dist\!\left(
\phi_{\#}\hat{\mu}_{\theta,T},
\phi_{\#}\hat{\mu}
\right)
+
\lambda\Omega(\pi_\theta).
\end{equation}
Choosing \(\phi\) to be the identity compares complete distributions of
PDs.  Choosing a lower-dimensional \(\phi\) instead
matches the induced distribution of a task specific topological
statistic.  Both are instances of the same persistence aware
distributional objective.

The framework also accommodates finite-horizon 
transformations.  Suppose a policy induces a map
\(C_\theta:\mathcal D_K\to\mathcal D_K\), 
%such as a denoising or
%compression procedure, 
and let \(\Delta\) be a
discrepancy between an input diagram and its transformed version.  A
paired objective is
\begin{equation}
\label{eq:paired_objective}
J_{\mathrm{pair}}(\theta)
=
\frac{1}{n}
\sum_{i=1}^{n}
\Delta\!\left(D_i,C_\theta(D_i)\right),
\end{equation}
possibly subject to a task constraint such as a prescribed diagram
cardinality.  
% The subject matched pairs
% \((D_i,C_\theta(D_i))\) define a coupling between the original empirical
% distribution and its pushforward
% \((C_\theta)_{\#}\hat{\mu}\). 
The paired objective
controls a transport based discrepancy between the original and
transformed populations.  
% This formulation connects
% topology preserving transformation tasks with the broader
% distributional framework.
Direct likelihood based learning on PD space is
generally difficult \citep{Nasrin_2024} because PDs are
variable cardinality objects in a non-Euclidean metric space without a
natural global coordinate system.  Equations~\eqref{eq:Jtheta}--%
\eqref{eq:paired_objective} instead permit learning through simulated
trajectories, empirical measures, and geometry aware discrepancies.

% \begin{figure}[h!]
% \centering
% \includegraphics[trim=-2.5cm 9cm 0cm 3cm,
% clip,
% width=7in,height= 2in]{figures/RL_PD_framework}

% \caption{Logical structure of the proposed framework. PDs evolve through topology aware edit operations controlled by an RL policy. The induced stochastic dynamics define a Markov kernel whose stationary law is matched to the empirical diagram distribution through a distributional objective.}
% \label{fig:framework_flow}

% \end{figure}

\subsection{Reward Function}
\label{sec:reward_function}

The reward is the optimization signal derived from the
distributional objective.
When the objective is evaluated at the trajectory or episode level, the
terminal reward is defined by
\begin{equation}
\label{eq:terminal_reward}
\mathcal R(\mathcal{T})
=
-\hat J_T(\theta),
\end{equation}
where \(\mathcal{T}=(D_0,a_0,\ldots,D_T)\) denotes the generated trajectory.
% Thus, maximizing expected reward is equivalent to minimizing the
% empirical distributional objective.
When the objective can be evaluated after individual edits, denser
feedback may be obtained from the stepwise improvement
\begin{equation}
\label{eq:incremental_reward}
r_t
=
J_t-J_{t+1},
\end{equation}
where \(J_t\) denotes the current value of the relevant 
objective.  
% An edit that decreases the objective receives positive
% reward, whereas an edit that increases it receives negative reward.  
In
an undiscounted finite-horizon episode, these rewards telescope:
$
\sum_{t=0}^{T-1}r_t
=
J_0-J_T.
$
Maximizing cumulative reward therefore minimizes the final objective
whenever \(J_0\) is fixed.  
% For discounted learning,
% the unmodified sum no longer telescopes exactly; one may therefore
% retain the terminal reward in Equation~\eqref{eq:terminal_reward} or use
% a discount-compatible potential-based shaping term when exact policy
% invariance is required.

For a discounted return with discount factor $\zeta\in(0,1)$, the
incremental reward $J_t-J_{t+1}$
does not telescope exactly. The simplest alternative is therefore to use
the terminal reward in Equation~\eqref{eq:terminal_reward}. If
intermediate feedback is desired, one may instead introduce a potential
function $\Phi:\mathcal D_K\to\mathbb R$ and add the shaping reward
\[
F_t
=
\zeta\Phi(D_{t+1})-\Phi(D_t).
\]
This construction changes the intermediate feedback while preserving
the optimal policies under the standard conditions for potential-based
reward shaping~\citep{ng1999policy}. For example, choosing
$\Phi(D)=-J(D)$ gives
$
F_t
=
J(D_t)-\zeta J(D_{t+1}).
$

Equations~\eqref{eq:terminal_reward} and
\eqref{eq:incremental_reward} provide two implementations of the same
principle.  Distributional objectives are naturally assigned as
trajectory level rewards, while diagram level fidelity or
transformation objectives may provide stepwise rewards.  The choice
depends on the scale at which the task objective can be estimated, but
the sign convention remains consistent throughout: the objective is
minimized and the corresponding reward is maximized.

\section{Theoretical Properties}
\label{sec:theory}

Because the proposed framework defines an implicit distribution through policy induced Markov dynamics rather than an explicit density, its theoretical validity requires a well-defined long run law and reliable trajectory based approximations. We establish irreducibility, aperiodicity, and drift and minorization conditions, yielding a unique stationary distribution and convergence independent of the initial diagram. We then prove convergence of empirical trajectory measures to this law, thereby justifying the Monte Carlo approximation used in the distributional objective and simulation-based policy optimization.

\subsection{Existence of Stationary Distribution}

To establish existence and uniqueness of the stationary distribution, we
verify four standard properties of the policy induced MC:
$\psi$-irreducibility, aperiodicity, a Foster–Lyapunov drift condition,
and a minorization condition on a small set. The technical statements and
their full proofs are deferred to Appendix~\ref{appendix:props}.

% \begin{theorem}[Existence and uniqueness of stationary distribution]
% \label{thm:stationary}
% Let $\{D_t\}_{t\ge0}$ be the MC on
% $(\mathcal D_K,\mathcal B(\mathcal D_K))$
% induced by policy $\pi_{\theta}$ and transition kernel $P_\theta$.
% Suppose the conditions verified in
% Appendix~\ref{appendix:props} hold: the chain is
% $\psi$-irreducible, aperiodic, and satisfies the stated
% Foster–Lyapunov drift and minorization conditions. Then there exists a
% unique stationary probability measure
% $
% \mu_\theta\in\mathcal P(\mathcal D_K)
% $
% such that
% \[
% \mu_\theta(B)
% =
% \int_{\mathcal D_K} P_\theta(B \mid D)d\mu_\theta(D),
% \qquad
% B\in\mathcal B(\mathcal D_K).
% \]
% Moreover, for any initial diagram $D_0\in\mathcal D_K$,  probability law (distribution) 
% $\mathcal L(D_t)$ of $D_t$ converges to $\mu_\theta$ in total variation:
% \[
% \|\mathcal L(D_t)-\mu_\theta\|_{\mathrm{TV}}
% \longrightarrow 0
% \qquad
% \text{as }t\to\infty,
% \]
% %where $\|\cdot\|_{\mathrm{TV}}$ denotes total variation distance between
% %probability measures.
% \end{theorem}

\begin{theorem}[Existence, uniqueness, and geometric ergodicity]
\label{thm:stationary}
Fix \(\theta\), and let \(\{D_t\}_{t\geq0}\) be the policy-induced
Markov chain on \(\mathcal D_K\) with transition kernel \(P_\theta\).
Suppose that the assumptions of
Propositions~\ref{prop:irreducible_app}--\ref{prop:minorization_app}
hold. Then the chain admits a unique stationary probability measure
\(\mu_\theta\).
Moreover, there exist constants \(M_V<\infty\) and
\(\varrho_V\in(0,1)\) such that
\[
\left\|
P_\theta^t(\,\cdot\mid D)-\mu_\theta
\right\|_{\mathrm{TV}}
\leq
M_VV(D)\varrho_V^t,
\qquad
D\in\mathcal D_K,\quad t\geq0,
\]
where \(P_\theta^t\) denotes the \(t\)-step transition kernel.
Consequently,
\[
\mathcal L(D_t)\longrightarrow\mu_\theta
\]
in total variation for every initial diagram.
\end{theorem}

\begin{proof}
Propositions~\ref{prop:irreducible_app}--%
\ref{prop:minorization_app} establish
\(\psi\)-irreducibility, aperiodicity, geometric Foster--Lyapunov
drift, and minorization on the corresponding small set. The geometric
drift theorem for general-state-space Markov chains
\citep{meyn2012markov} therefore gives positive Harris recurrence and
geometric ergodicity. Positive Harris recurrence guarantees existence
of an invariant probability measure, and \(\psi\)-irreducibility
guarantees its uniqueness. The stated geometric total-variation bound
and convergence follow.
\end{proof}

% \begin{proof}
% Fix $\theta$. 
% % By the construction in Section~\ref{sec:policy_mc}, the
% % policy $\pi_\theta$  defines a MC
% % $\{D_t\}_{t\ge0}$ on the measurable state space
% % $(\mathcal D_K,\mathcal B(\mathcal D_K))$ with transition kernel
% % $P_{\theta}$.
% Appendix~\ref{appendix:props} verifies that this chain is
% $\psi$-irreducible and aperiodic, and satisfies the required
% Foster–Lyapunov drift and minorization conditions. By the standard
% geometric ergodicity theorem for general-state-space MC
% \citep{meyn2012markov}, these conditions imply positive Harris
% recurrence and geometric ergodicity. Positive Harris recurrence gives the
% existence of an invariant probability measure, and
% $\psi$-irreducibility implies that this invariant probability measure is
% unique. We denote it by $\mu_\theta$.
% The invariance property of $\mu_\theta$ is precisely
% \[
% \mu_\theta(B)
% =
% \int_{\mathcal D_K} P_\theta(B \mid D)d\mu_\theta(D),
% \qquad
% B\in\mathcal B(\mathcal D_K).
% \]
% Geometric ergodicity further implies convergence of the law of the chain
% to the stationary distribution in total variation. Therefore, for every
% initial diagram $D_0\in\mathcal D_K$,
% \[
% \|\mathcal L(D_t)-\mu_\theta\|_{\mathrm{TV}}
% \to0
% \qquad
% \text{as }t\to\infty.
% \]
% This proves existence, uniqueness, invariance, and convergence to the
% stationary distribution.
% \end{proof}
%\subsection{Convergence of the Law}

Theorem~\ref{thm:stationary} gives total variation convergence to the stationary law. Because the learning objectives use geometric discrepancies on PD space, the following corollary states the corresponding weak convergence and, under an additional second moment condition, convergence in the Wasserstein-\(2\) distance \(W_2\). 
% For the Wasserstein convergence results below, we assume that
% \((\mathcal D_K,d_{\mathcal D})\) is a Polish metric space. 

Equipped with the \(p\)-Wasserstein metric, the space of PDs is complete and separable, and hence Polish~\citep{mileyko2011probability}. Since \(K\subset\mathbb W\) is compact and separated from the diagonal, the restricted space \(\mathcal D_K\) is a closed subspace and is therefore also Polish.
% The space of PDs equipped with \(p\)-Wasserstein metric is complete and separable~\citep{mileyko2011probability}.  The restricted state space \(\mathcal D_K\) inherits these properties.  Indeed, because \(K\) is a compact subset of the strict wedge
% \(\mathbb W=\{(b,d)\in\mathbb R^2:b<d\}\), it is separated from the diagonal \(\Delta\): there exists
% \[
% \eta_K
% =
% \inf_{x\in K}d(x,\Delta)
% >0.
% \]
% Diagrams of different cardinalities are therefore separated by a positive distance, so every Cauchy sequence in \(\mathcal D_K\) eventually lies in a fixed-cardinality stratum
% \[
% \mathcal D_K^k
% =
% \{D\in\mathcal D_K:|D|=k\}.
% \]
% Each \(\mathcal D_K^k\) is the permutation quotient of the compact space \(K^k\) and is consequently compact.  It follows that \(\mathcal D_K\) is complete.  Moreover, since it is a countable union of separable fixed-cardinality strata, it is separable.  Hence \((\mathcal D_K,d_{\mathcal D})\) is a Polish metric space.
As in
Section~\ref{sec:learning_objective}, any representation space
\((\mathcal Z,d_{\mathcal Z})\) is also assumed to be Polish.

\begin{corollary}[Convergence of the law]
\label{cor:law_w2}
Under the assumptions of Theorem~\ref{thm:stationary},
$
\mathcal L(D_t)\Rightarrow\mu_\theta$ as $t\to\infty.
$
If, in addition, their second moments satisfy, for some
\(D_0\in\mathcal D_K\),
\[
\int_{\mathcal D_K}
d_{\mathcal D}(D,D_0)^2\,\mathcal L(D_t)(dD)
\longrightarrow
\int_{\mathcal D_K}
d_{\mathcal D}(D,D_0)^2\,\mu_\theta(dD)
<\infty,
\]
then
\[
W_2\!\left(\mathcal L(D_t),\mu_\theta\right)
\longrightarrow 0.
\]
\end{corollary}

\begin{proof}
Theorem~\ref{thm:stationary} gives
\[
\left\|\mathcal L(D_t)-\mu_\theta\right\|_{\mathrm{TV}}
\longrightarrow 0
\qquad\text{as }t\to\infty.
\]
Total variation convergence implies weak convergence; hence
\[
\mathcal L(D_t)\Rightarrow\mu_\theta
\qquad\text{as }t\to\infty,
\] The assumed convergence
of second moments and the standard characterization of Wasserstein
convergence~\citep{villani2009optimal} then give
\[
W_2\!\left(\mathcal L(D_t),\mu_\theta\right)
\longrightarrow 0
\qquad\text{as }t\to\infty.
\]
\end{proof}

\subsection{Empirical Measure Convergence}

Because the stationary distribution is not available in closed form,
we approximate it using samples from a policy-generated trajectory.

\begin{theorem}[Empirical and pushforward measure convergence]
\label{thm:empirical}
Under the assumptions of Theorem~\ref{thm:stationary}, define
$
\hat{\mu}_{\theta,T}
=
\frac{1}{T}
\sum_{t=1}^{T}\delta_{D_t}$.
Then
\[
\hat{\mu}_{\theta,T}
\Rightarrow
\mu_\theta
\qquad\text{almost surely as }T\to\infty.
\]
If, for some \(D_0\in\mathcal D_K\),
\[
\int_{\mathcal D_K}
d_{\mathcal D}(D,D_0)^2\,\mu_\theta(dD)
<\infty,
\]
then
$
W_2\!\left(\hat{\mu}_{\theta,T},\mu_\theta\right)
\longrightarrow 0$ almost surely.
More generally, let
\(\phi:\mathcal D_K\to\mathcal Z\) be measurable.  If, for some
\(z_0\in\mathcal Z\),
\[
\int_{\mathcal D_K}
d_{\mathcal Z}\!\left(\phi(D),z_0\right)^2
\,\mu_\theta(dD)
<\infty,
\]
then $
W_{2,\mathcal Z}
\left(
\phi_{\#}\hat{\mu}_{\theta,T},
\phi_{\#}\mu_\theta
\right)
\longrightarrow 0$ almost surely.

\end{theorem}

\begin{proof}

Because \(\mathcal D_K\) is Polish, it admits a countable
convergence-determining family of bounded continuous functions
\(\{f_j\}_{j\geq1}\). For every \(j\), the Markov chain ergodic
theorem~\citep{meyn2012markov} gives
\[
\int_{\mathcal D_K}f_j(D)\,
\hat{\mu}_{\theta,T}(dD)
\longrightarrow
\int_{\mathcal D_K}f_j(D)\,
\mu_\theta(dD)
\qquad\text{almost surely}.
\]
By countability, these limits hold simultaneously on a common
probability-one event, and therefore
\[
\hat{\mu}_{\theta,T}\Rightarrow\mu_\theta
\qquad\text{almost surely}.
\]
Under the stated second-moment condition, applying the ergodic theorem
to \(D\mapsto d_{\mathcal D}(D,D_0)^2\) also gives convergence of
second moments. The characterization of Wasserstein
convergence~\citep{villani2009optimal} then yields $W_2\!\left(\hat{\mu}_{\theta,T},\mu_\theta\right)
\longrightarrow0$ almost surely.

For the pushforward result, choose a countable convergence-determining
family of bounded continuous functions \(\{g_j\}_{j\geq1}\) on
\(\mathcal Z\). Since \(\phi\) is measurable, each \(g_j\circ\phi\) is
bounded and measurable. The same ergodic argument gives
\[
\phi_{\#}\hat{\mu}_{\theta,T}
\Rightarrow
\phi_{\#}\mu_\theta
\qquad\text{almost surely}.
\]
Applying the ergodic theorem to
\(D\mapsto d_{\mathcal Z}(\phi(D),z_0)^2\) gives convergence of the
corresponding second moments, and hence
\[
W_{2,\mathcal Z}
\left(
\phi_{\#}\hat{\mu}_{\theta,T},
\phi_{\#}\mu_\theta
\right)
\longrightarrow0
\qquad\text{almost surely}.
\]
\end{proof}

Theorem~\ref{thm:empirical} shows that finite policy-generated
trajectories consistently approximate both the stationary law and its
task-specific pushforward.  The following proposition transfers this
convergence to the empirical learning objective.

\begin{proposition}[Consistency of the empirical distributional objective]
\label{prop:objective_consistency}
Fix \(\theta\), and let
\(\phi:\mathcal D_K\to\mathcal Z\) be the measurable representation
defined in Section~\ref{sec:learning_objective}.  Suppose that, for some
\(z_0\in\mathcal Z\),
\[
\int_{\mathcal D_K}
d_{\mathcal Z}\!\left(\phi(D),z_0\right)^2
\,\mu_\theta(dD)
<\infty.
\]
Define
$
J(\theta)
=
W_{2,\mathcal Z}
\left(
\phi_{\#}\mu_\theta,
\phi_{\#}\hat{\mu}
\right)$
and
$
\hat J_T(\theta)
=
W_{2,\mathcal Z}
\left(
\phi_{\#}\hat{\mu}_{\theta,T},
\phi_{\#}\hat{\mu}
\right)$. 
Then
\[
\hat J_T(\theta)
\longrightarrow
J(\theta)
\qquad\text{almost surely as }T\to\infty.
\]
\end{proposition}

\begin{proof}
By Theorem~\ref{thm:empirical},
\[
W_{2,\mathcal Z}
\left(
\phi_{\#}\hat{\mu}_{\theta,T},
\phi_{\#}\mu_\theta
\right)
\longrightarrow 0
\qquad\text{almost surely}.
\]
The reverse triangle inequality for \(W_{2,\mathcal Z}\) gives
\[
\left|
\hat J_T(\theta)-J(\theta)
\right|
\leq
W_{2,\mathcal Z}
\left(
\phi_{\#}\hat{\mu}_{\theta,T},
\phi_{\#}\mu_\theta
\right).
\]
The right-hand side converges to zero almost surely, proving the result.
\end{proof}
Together, the stationary-law, ergodic, and objective consistency results justify learning policies from simulated trajectories. We next discuss the policy representations, topology aware proposals, task specific objectives, and simulation based optimization required for implementation.

%%%%%%%%%%%%%%%%%%%%%%%%%%%%%%%%%%%%%%%%%%%%%%%%
\section{Approximation and Algorithmic Considerations}
\label{sec:algorithms}

% The theoretical framework developed in Sections~\ref{sec:framework}
% and~\ref{sec:theory} provides a principled formulation of
% policy controlled stochastic dynamics on PD space.
We now describe the practical realization of the policy controlled stochastic dynamics on PD space  for simulation and learning.
Because the stationary distribution and the distributional objective are
not available in closed form, training must rely on finite simulated
trajectories, empirical discrepancy estimation, and iterative policy
optimization.
Figure~\ref{fig:algorithmic_framework} summarizes the computational
workflow. 
% Starting from observed PDs, we construct an
% empirical target distribution and iteratively simulate policy controlled
% diagram trajectories using topology aware edit operations. The generated
% empirical distribution is compared with the target distribution through
% a geometry aware discrepancy, and the resulting feedback is used to
% update the policy parameters.

\begin{figure}[t]
\centering
\includegraphics[width=\textwidth]{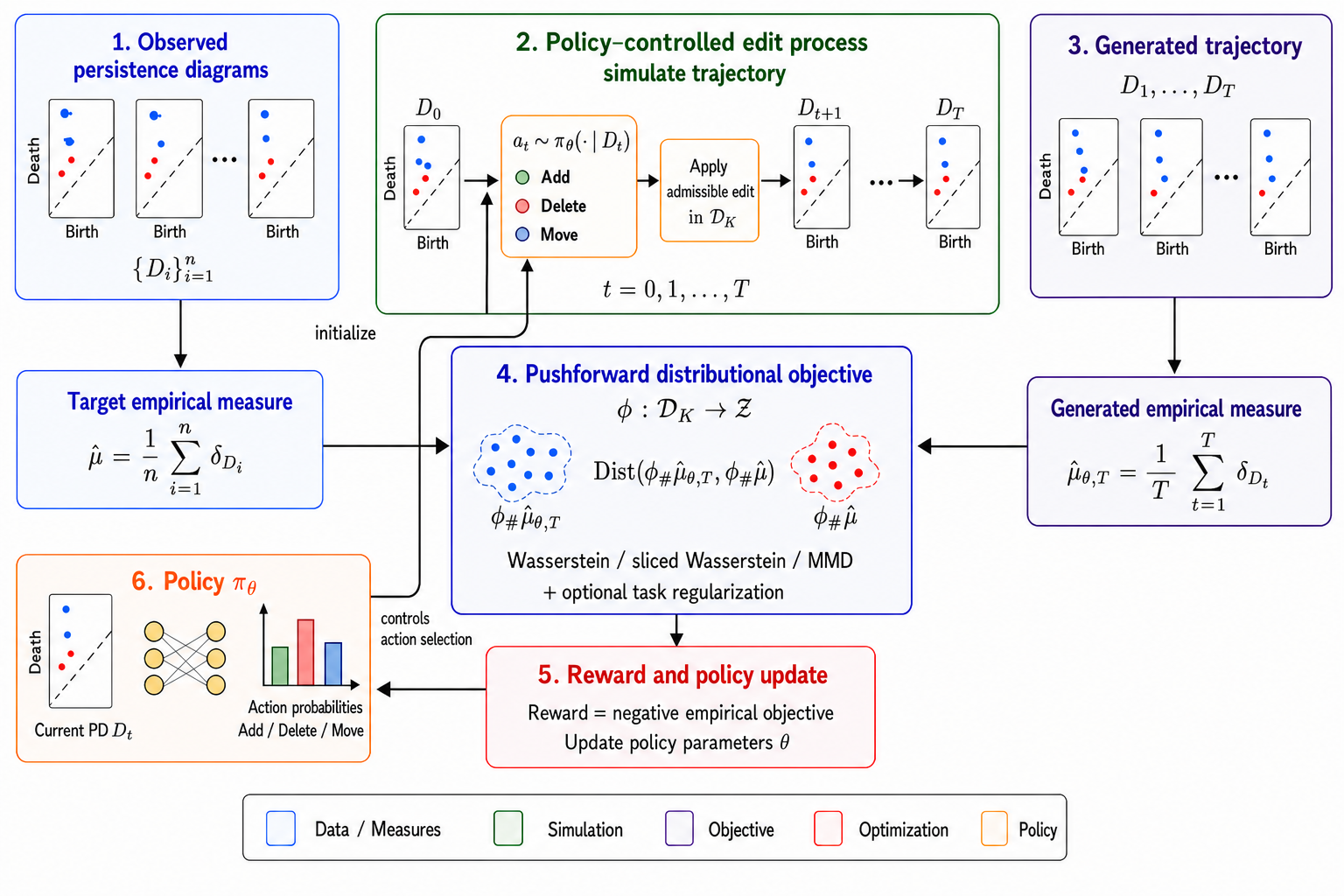}
\caption{Algorithmic workflow for policy controlled PD generation. 
% Observed PDs define an empirical target distribution. A parameterized policy generates trajectories through topology-aware edit operations. After burn-in, the generated empirical distribution is compared with the target distribution using a geometry aware discrepancy, and the resulting loss is used to update the policy parameters.
}
\label{fig:algorithmic_framework}
\end{figure}

\subsection{Policy Conditioning for Variable-Cardinality Diagrams}
\label{sec:policy_representation}

A practical challenge in implementing RL on PD
space is that the number of points in a diagram may vary across states
and along a trajectory. Consequently, a PD cannot be passed directly
to a policy model that expects an input vector of fixed dimension. 
% This
% is a computational issue rather than a restriction of the underlying
% stochastic framework: the state remains the complete diagram
% $D\in\mathcal D_K$, and the transition kernel continues to act directly
% on PD space.
One practical solution is to introduce a measurable policy feature map
\[
\xi:\mathcal D_K\longrightarrow\mathbb R^m
\]
that converts each diagram into a fixed-dimensional vector. Possible
components of $\xi(D)$ include diagram cardinality; summaries of birth,
death, and persistence values; total and maximum persistence;
persistence entropy; counts above prescribed persistence thresholds;
and finite-dimensional summaries of Betti curves or other stable
topological representations relevant to the task. 
% The particular features should be chosen
% according to the actions and structural information relevant to the
% task.

In a feature based implementation, the policy uses $\xi(D)$ to compute
action probabilities and, when necessary, parameters of the associated
proposal distributions. We continue to write the policy as
$\pi_\theta(a\mid D)$, with the understanding that its computational
dependence on $D$ may be mediated through $\xi(D)$. Because $\xi$ is a
deterministic function of the current diagram, this construction still
defines a Markov policy on $\mathcal D_K$. The feature vector affects
how the policy selects actions, but it does not replace the PD as the
state of the stochastic process.
A fixed-dimensional feature map is not the only possible
implementation. Permutation-invariant set networks, attention based
models, and other variable cardinality architectures may be used to
condition the policy directly on the points of a diagram. 
% The proposed
% framework therefore requires only a well-defined measurable policy on
% PD space, not a particular vector representation or policy
% architecture.
The policy-feature map $\xi$ is distinct from the map $\phi$ used in
the distributional objective. The map $\xi$ provides computational
inputs for action selection, whereas $\phi$ specifies the
persistence aware quantities whose pushforward distributions are
compared during learning.
\subsection{Topology Aware Proposal Mechanisms}
\label{sec:topology_aware_proposals}

The edit operations introduced in Section~\ref{sec:framework} may be
implemented using proposal mechanisms that explicitly exploit the
geometry of PD space. Rather than relying on uniform
or isotropic proposals, we use topology aware mechanisms that bias the
stochastic dynamics toward realistic diagram configurations while
preserving dominant structure. 
The proposal mechanisms  are summarized in Table~\ref{tab:proposal_mechanisms}.

% \paragraph{Topology-aware add proposals.}
% New points are proposed in birth–persistence coordinates
% $
% z=(b,p)$, and $p>0$,
% using an empirical proposal distribution estimated from observed
% training diagrams. To separate topological noise from meaningful
% persistent structure, we use a two-regime mixture model:
% \[
% q_{\mathrm{add}}(z)
% =
% \pi_{\mathrm{noise}} q_{\mathrm{noise}}(z)
% +
% \pi_{\mathrm{signal}} q_{\mathrm{signal}}(z),
% \]
% where $\pi_{\mathrm{noise}}+\pi_{\mathrm{signal}}=1.$
% Here, $q_{\mathrm{noise}}$ models low-persistence features near the
% diagonal, while $q_{\mathrm{signal}}$ models higher-persistence
% features corresponding to stable topological structure. 
% \paragraph{Persistence-aware deletion.}
% Uniform deletion can remove highly informative topological features.
% Instead, each point is assigned a deletion probability
% \[
% w_i \propto \exp(-\beta p_i),
% \qquad \beta>0.
% \]
% Thus, low-persistence features near the diagonal are more likely to be
% removed, while highly persistent features are preserved.

% \paragraph{Geometry-aware move proposals.}
% To perturb existing points while respecting diagram geometry, move
% operations are performed in birth–persistence coordinates:
% \[
% (b_i,p_i)
% \mapsto
% (b_i+\varepsilon_b,; p_i+\varepsilon_p),
% \]
% where
% \[
% (\varepsilon_b,\varepsilon_p)
% \sim
% \mathcal N(0,\Sigma).
% \]
% The covariance matrix (\Sigma) may be anisotropic, allowing different
% scales of perturbation in birth and persistence directions. This gives
% finer control over perturbations near the diagonal and in
% high-persistence regions.

\begin{table}[h!]
\centering
\caption{Topology-aware proposal mechanisms for policy-controlled edit operations.}
\label{tab:proposal_mechanisms}
\renewcommand{\arraystretch}{1.35}
\begin{tabularx}{\textwidth}{
|>{\centering\arraybackslash}p{2.1cm}
|>{\centering\arraybackslash}p{4.2cm}
|X|}
\hline
\textbf{Edit} & \textbf{Proposal Rule} & \textbf{Purpose} \\
\hline

{Add}
&
$
q_{\mathrm{add}}(z)
=
\pi_{\mathrm{noise}} q_{\mathrm{noise}}(z)
+
\pi_{\mathrm{signal}} q_{\mathrm{signal}}(z)
$
&
New PD points are sampled from a two-regime mixture separating near diagonal noise from high persistence signal. 
\\
\hline

{Delete}
&
$
w_i \propto \exp(-\beta p_i)
$
&
Deletion probabilities are inversely weighted by persistence $p_i = d_i-b_i$, making low persistence features more likely to be removed while preserving stable topological structure.
\\
\hline

{Move}
&
$
(b_i,p_i)\mapsto
(b_i+\varepsilon_b,\,
p_i+\varepsilon_p)
$
&
Existing points are perturbed using anisotropic Gaussian noise, allowing different scales of movement in birth and death directions.
\\
\hline
\end{tabularx}
\end{table}

% Together, these topology-aware proposals improve simulation realism by
% encouraging exploration in regions supported by observed data while
% suppressing unstable topological noise.

\subsection{Task-Specific Objective Specification}
\label{sec:task_objective}

Section~\ref{sec:learning_objective} defines distributional learning
through the representation $\phi$ and the discrepancy $\Dist$.  The map $\phi$
determines which properties of a PD are compared, while $\Dist$
determines how differences between the resulting probability measures
are quantified. Taking $\phi$ to be the identity map compares complete
PD distributions. Alternatively, $\phi$ may return selected
persistence aware quantities, such as a vector of persistence
statistics, a discretized Betti curve, a diagram-complexity measure, or
a rare event indicator. Several characteristics can be considered
simultaneously by allowing $\phi(D)$ to contain multiple components.

% For a trajectory $D_1,\ldots,D_T$, the empirical objective is
% \[
% \hat J_T(\theta)
% =
% \Dist\!\left(
% \phi_{\#}\hat{\mu}_{\theta,T},
% \phi_{\#}\hat{\mu}
% \right),
% \qquad
% \hat{\mu}_{\theta,T}
% =
% \frac{1}{T}
% \sum_{t=1}^{T}\delta_{D_t}.
% \]
The task specific objective design does not require introducing
separate global, structural, and application dependent losses. These
requirements can instead be expressed through the choice of $\phi$,
the metric structure of $\mathcal Z$, and the discrepancy $\Dist$.
When necessary, the optional regularization term
$\Omega(\pi_\theta)$ introduced in
Section~\ref{sec:learning_objective} may be used to account for
requirements that are not naturally expressed through distributional
matching.

% The map $\phi$ used here has a different role from the policy-feature
% map $\xi$ introduced in Section~\ref{sec:policy_representation}.
% The map $\xi$ provides computational inputs for action selection,
% whereas $\phi$ determines the quantities compared in the
% distributional objective. The two maps need not have the same
% codomain, components, or level of resolution.

\subsection{Simulation Based Policy Optimization}
\label{sec:policy_optimization}

The stationary distribution $\mu_\theta$ is generally unavailable in
closed form, and edit operations can produce discrete changes in both
the locations and the number of PD points. Consequently, the objective
must be approximated using trajectories generated by the current
policy. Although the states along a trajectory are dependent rather
than independent samples, the ergodic results developed in
Section~\ref{sec:theory} justify the use of trajectory averages for
sufficiently long runs.
For fixed policy parameters $\theta$, the policy controlled dynamics
generate $D_1,\ldots,D_T$ and hence the empirical measure
$\hat{\mu}_{\theta,T}$. The discrepancy between
$\phi_{\#}\hat{\mu}_{\theta,T}$ and $\phi_{\#}\hat{\mu}$ provides
trajectory-level feedback. 
% A terminal reward may be defined as the
% negative empirical objective, so maximizing expected reward corresponds
% to reducing the distributional discrepancy. When stepwise feedback is
% preferred, the same objective may be decomposed or used to construct a
% consistent reward-shaping scheme, as described in the reward section.

The appropriate optimization method depends on the policy
parameterization and the action mechanism. Likelihood ratio
policy gradient estimators and actor--critic methods can be used when
the stochastic policy is differentiable with respect to $\theta$.
Value-based or gradient-free methods may be preferable when actions are
discrete, proposal mechanisms are not differentiable, or evaluating
$\Dist$ is computationally expensive. The framework therefore does not
require one particular RL algorithm.

Several practical issues must be handled during optimization. The set
of admissible actions may depend on the current diagram; for example,
deletion is unavailable for the empty diagram, and all proposed points
or movements must remain in $K$. Invalid actions should therefore be
masked or assigned well-defined transition behavior. Exploration must
also be sufficient to visit different regions of PD space without
placing excessive probability on uninformative configurations.
Exploration schedules, policy-entropy regularization, multiple
trajectories, and batched discrepancy estimates may be used to improve
stability and reduce simulation variability.

Algorithm~\ref{alg:rl_pd} summarizes the general training procedure.

\begin{algorithm}[H]
\caption{Policy controlled learning on PD space}
\label{alg:rl_pd}
\begin{algorithmic}[1]

\Require Observed PDs $\{D_i\}_{i=1}^{n}$, initial policy
$\pi_\theta$, trajectory length $T$, measurable representation
$\phi$, and discrepancy $\Dist$
\State Form the empirical target measure
$\displaystyle
\hat{\mu}=\frac{1}{n}\sum_{i=1}^{n}\delta_{D_i}$.
\While{the stopping criterion is not satisfied}
\State Initialize $D_0\in\mathcal D_K$.
\For{$t=0,1,\ldots,T-1$}
\State Evaluate the policy at $D_t$, optionally using the policy
features $\xi(D_t)$.
\State Sample an admissible action
$A_t\sim\pi_\theta(\,\cdot\mid D_t)$.
\State Generate the next diagram according to
$D_{t+1}\sim P(\,\cdot\mid D_t,A_t)$.
\EndFor
\State Form the trajectory empirical measure
$\displaystyle
\hat{\mu}_{\theta,T}
=
\frac{1}{T}\sum_{t=1}^{T}\delta_{D_t}$.
\State Evaluate
$\displaystyle
\hat J_T(\theta)
=
\Dist\!\left(
\phi_{\#}\hat{\mu}_{\theta,T},
\phi_{\#}\hat{\mu}
\right)$.
\State Construct reward feedback from
$-\hat J_T(\theta)$ and update $\theta$ using the selected
RL method.
\EndWhile

\Return The optimized policy $\pi_\theta$.
\end{algorithmic}
\end{algorithm}

% If optional policy regularization is required, the empirical objective
% in Algorithm~\ref{alg:rl_pd} is replaced by
% \[
% \hat J_T(\theta)
% =
% \Dist\!\left(
% \phi_{\#}\hat{\mu}_{\theta,T},
% \phi_{\#}\hat{\mu}
% \right)
% +
% \lambda\Omega(\pi_\theta).
% \]

% \subsection{Connection to the Experiments}

% The same algorithmic structure supports the experimental studies
% reported in Section~\ref{sec:experiments}. In the synthetic setting, the
% goal is to test whether the policy-induced dynamics can recover
% distributional features such as rare high-persistence events. In the
% real-data settings, the same framework is used to learn simplified
% topological representations that preserve dominant persistent structure
% while reducing unstable near-diagonal features. Thus, the experiments
% differ primarily in the choice of input diagrams, reward or
% regularization terms, and evaluation metrics, while sharing the same
% policy-controlled diagram-space dynamics.

% \section{Algorithmic Realization}

% 5.1 Policy Representation and State Features

% 5.2 Topology-Aware Proposal Mechanisms

% 5.3 Trajectory Simulation and Empirical Measure Estimation

% 5.4 Simulation-Based Policy Optimization

%%%%%%%%%%%%%%%%%%%%%%%%%%%%%%%%%%%%%%%%%%%%%%%%
\section{Experiments}
\label{sec:experiments}

We evaluate the proposed framework on one synthetic and one real-data
case studies designed to examine how learned distributions on
PD space can support different scientific objectives.
The
experiments demonstrate how distributional learning enables inference on
rare topological events, construction of topology preserving
representations, and characterization of heterogeneous topological
structure across complex datasets.
% The synthetic experiment provides a controlled setting for evaluating
% whether the learned policy can recover rare but topologically
% significant events that are difficult to capture under conventional
% stochastic generators. The real-data experiments assess the ability of
% the framework to learn parsimonious yet informative topological
% representations from complex high-dimensional data while preserving
% dominant persistent structure. 
Across all experiments, the underlying
diagram space dynamics remain unchanged; differences arise primarily
through the choice of state summaries, task specific objective terms,
and evaluation criteria.

\subsection{Synthetic Rare Event Recovery}
\label{sec:exp_rare}

We first consider a controlled synthetic experiment designed to test
whether policy ontrolled PD dynamics can recover a low probability
topological event. The purpose of this experiment is to isolate the
task specific distributional setting introduced in
Section~\ref{sec:learning_objective}: the representation $\phi$ is
chosen to retain a scientifically relevant rare event indicator, and
the policy is trained to match its pushforward distribution.

\subsubsection{Synthetic data and rare event definition.}
We simulated \(n=400\) planar point clouds from two regimes. In the common regime, three Gaussian clusters are centered near the vertices of a triangle. In the rare regime, noisy bridge points are added along the triangle edges, increasing the likelihood of a persistent one-dimensional cycle while preserving the overall clustered shape (Figure~\ref{fig:rare_synthetic_setup}, upper row). Each point cloud was independently assigned to the rare regime with probability \(0.10\); the simulated sample contained \(49\) rare-regime clouds, or \(49/400=0.1225\). We then computed the corresponding PDs and applied the topological rare-event criterion defined below (Figure~\ref{fig:rare_synthetic_setup}, lower row). This criterion depends on the resulting PD rather than the generating regime label. It classified \(48\) diagrams as rare, giving an empirical rare-event probability of \(48/400=0.120\).

For each point cloud, we computed its finite $H_1$ PD using a
Vietoris--Rips filtration. 
% For a diagram $D$, define its maximum
% persistence by
% \[
% p_{\max}(D)
% =
% \max_{(b,d)\in D}(d-b),
% \]
% with $p_{\max}(\varnothing)=0$. 
We selected
\[
\tau
=
\max\left\{
0.15,\,
Q_{0.88}\bigl(
\{p_{\max}(D_i)\}_{i=1}^{n}
\bigr)
\right\},
\]
where $Q_{0.88}$ denotes the empirical $0.88$ quantile and $p_{\max}(D)=
\max_{(b,d)\in D}(d-b)$ is the maximum persistence point in a persistence diagram $D$. This gave
$\tau=0.2402$. The rare-event indicator is
\[
R_\tau(D)
=
\mathbf 1\{p_{\max}(D)>\tau\}.
\]
The empirical target probability was therefore
\[
\omega_{\mathrm{data}}
=
\int_{\mathcal D_K}R_\tau(D)\,\hat{\mu}(dD)
=
\frac{1}{n}
\sum_{i=1}^{n}R_\tau(D_i)
=
0.120.
\]
In the notation of Section~\ref{sec:learning_objective}, the primary
task specific representation is
\[
\phi_{\mathrm{syn}}(D)=R_\tau(D),
\qquad
\phi_{\mathrm{syn}}:\mathcal D_K\longrightarrow\{0,1\}.
\]
Consequently, $(\phi_{\mathrm{syn}})_{\#}\hat{\mu}$ is a Bernoulli
distribution with success probability $\omega_{\mathrm{data}}$. Matching
the success probability therefore matches the complete pushforward
distribution on $\{0,1\}$.

\begin{figure}[t]
\centering
\includegraphics[width=0.65\linewidth]
{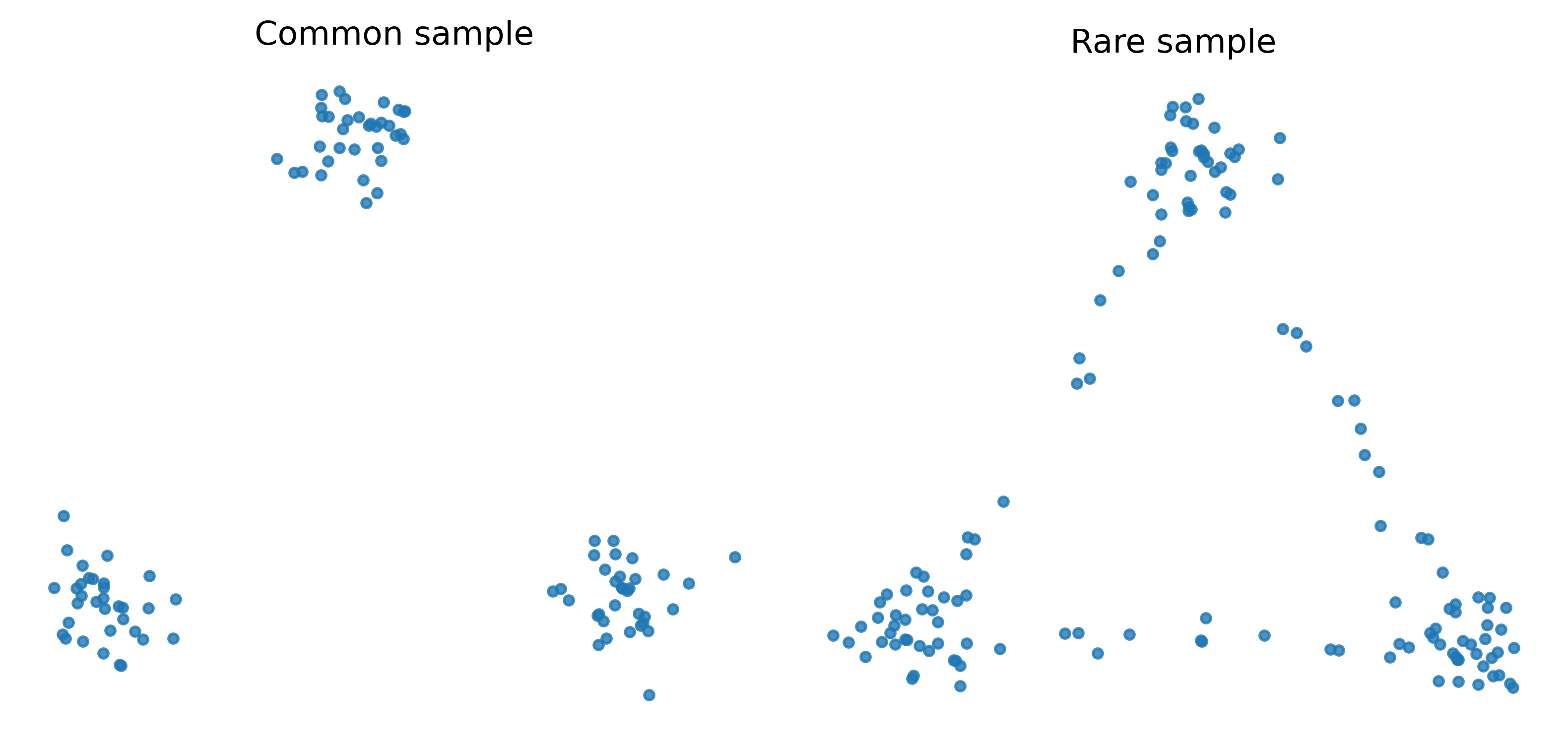}

\vspace{0.4em}

\includegraphics[width=0.75\linewidth]
{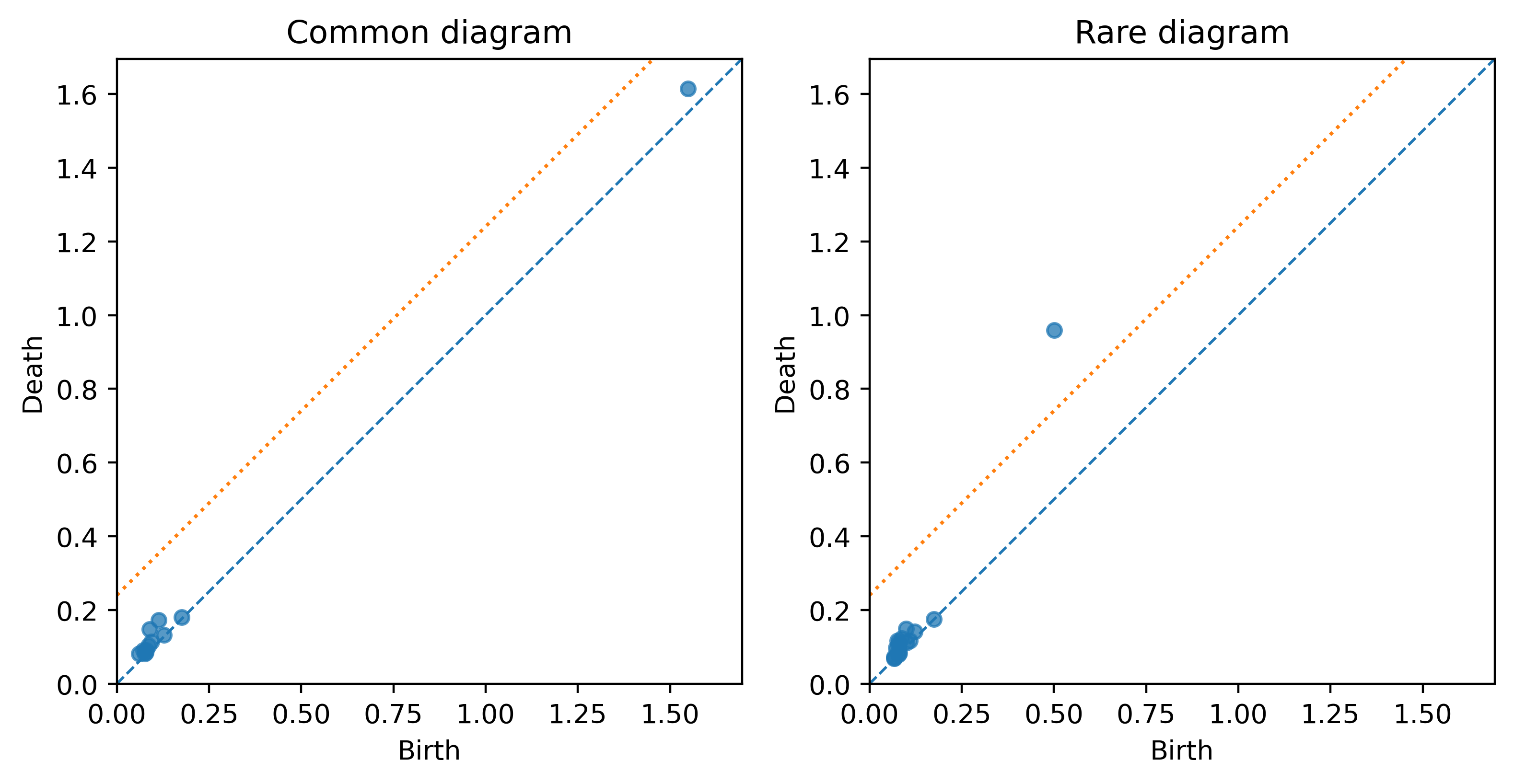}

\caption{Synthetic construction for rare-event recovery.
The upper row shows representative point clouds from the common and
rare regimes. The lower row shows their corresponding $H_1$ PDs. In
the PD panels, the dashed line is the diagonal and the dotted line is
$d=b+\tau$, above which a point has persistence greater than $\tau$.}
\label{fig:rare_synthetic_setup}
\end{figure}

\subsubsection{Policy controlled edit process.}
The state of the controlled process is the current PD. For policy
conditioning, each diagram is represented by a ten-dimensional feature
vector $\xi(D)$ containing scaled cardinality; the mean, standard
deviation, and maximum of persistence; the fraction and number of
points with persistence greater than $\tau$; the mean and standard
deviation of birth; the mean death; and the indicator $R_\tau(D)$.
As discussed in Section~\ref{sec:policy_representation}, these features
are computational inputs to the policy and do not replace $D$ as the
state.
The policy selects among six edit actions:
$
\mathcal A
=
\{
\texttt{add\_low},
\texttt{add\_high},
\texttt{delete\_low},
\texttt{delete\_random},
\texttt{move\_small},
\texttt{move\_large}
\}.
$
The two add actions sample, with small jitter, from empirical pools of
PD points whose persistence \(d-b\) is at most or greater than \(\tau\), respectively. The deletion actions
either favor low persistence points or select a point uniformly. The
move actions perturb a uniformly selected point at one of two scales in
birth--death coordinates. All edited points are returned to the
admissible region $K$. A two-hidden-layer policy network maps
$\xi(D)$ to a categorical distribution over these six actions.

\subsubsection{Empirical objective and reward.}

% The factor $5$ introduces an asymmetric penalty for overshooting the
% target probability. It is an implementation-specific stabilization
% choice rather than a requirement of the general framework.

For a retained training trajectory $D_1,\ldots,D_T$, let
\[
\hat \omega_{\theta,T}
=
\int_{\mathcal D_K}
R_\tau(D)\,\hat{\mu}_{\theta,T}(dD)
=
\frac{1}{T}
\sum_{t=1}^{T}R_\tau(D_t).
\]
The squared difference
$(\hat \omega_{\theta,T}-\omega_{\mathrm{data}})^2$ compares the empirical
pushforward distributions through their Bernoulli parameters. In fact,
for the linear kernel on $\{0,1\}$, this difference is squared MMD\@.
We therefore define the primary task specific discrepancy by
\[
\Dist_{\mathrm{rare}}\!\left(
(\phi_{\mathrm{syn}})_{\#}\hat{\mu}_{\theta,T},
(\phi_{\mathrm{syn}})_{\#}\hat{\mu}
\right)
=
(\hat \omega_{\theta,T}-\omega_{\mathrm{data}})^2.
\]
Let $\delta_T
=
\hat{\omega}_{\theta,T}
-
\omega_{\mathrm{data}}.$
During training, we used an asymmetric version of this discrepancy to
discourage overestimation of the rare-event frequency:
\[
\ell_{\mathrm{rare}}(\delta_T)
=
\left(
1+4\,\mathbf{1}_{\{\delta_T>0\}}
\right)\delta_T^2.
\]
Thus, positive deviations receive five times the squared-error penalty
of negative deviations.
The implementation also included regularization for mean maximum
persistence, mean cardinality, and diagram complexity. Define
$
\hat m_{\max,T}
=
\frac{1}{T}\sum_{t=1}^{T}p_{\max}(D_t),$ 
$m_{\max,\mathrm{data}}
=
\frac{1}{n}\sum_{i=1}^{n}p_{\max}(D_i),
$
and
$
\hat m_{N,T}
=
\frac{1}{T}\sum_{t=1}^{T}|D_t|$, 
$m_{N,\mathrm{data}}
=
\frac{1}{n}\sum_{i=1}^{n}|D_i|.
$

Let
$\lambda_{\mathrm{rare}},
\lambda_{\max},
\lambda_{\mathrm{card}},
\lambda_{\mathrm{complexity}}\geq0$
be user-specified weights controlling the overshoot penalty and the
relative contributions. Define the
empirical regularization by
\[
\hat\Omega_T^{\mathrm{syn}}(\theta)
={}
\lambda_{\max}\, \bigl(\hat m_{\max,T}-m_{\max,\mathrm{data}}\bigr)^2 + \lambda_{\mathrm{card}}\,
\bigl(\hat m_{N,T}-m_{N,\mathrm{data}}\bigr)^2
+ \lambda_{\mathrm{complexity}}\,\hat m_{N,T}.
\]
The complete empirical objective is therefore
\[
\widehat J_T^{\mathrm{syn}}(\theta)
=
\lambda_{\mathrm{rare}}
\ell_{\mathrm{rare}}(\delta_T)
+
\widehat\Omega_T^{\mathrm{syn}}(\theta).
\]
The corresponding terminal reward was
$
r_{\mathrm{terminal}}
=
-\hat J_T^{\mathrm{syn}}(\theta).
$
Thus, the dominant component of the reward directly targets the
pushforward distribution of the rare event indicator, while the
optional regularization provides weaker control of persistence
magnitude and diagram size.

\subsubsection{Training and reference generator.}
The policy was trained for $350$ epochs with $10$ simulated episodes
per epoch. Each episode contained $55$ edits. The first $15$ edited
states were treated as burn-in and excluded, and the remaining $T=40$
states were reindexed as $D_1,\ldots,D_T$ when constructing
$\hat{\mu}_{\theta,T}$. 
Policy parameters were updated with a
REINFORCE estimator using a moving reward baseline.
% , an entropy weight of
% $0.02$, and Adam with learning rate $5\times10^{-4}$. 
Every $25$
epochs, the current policy was evaluated using $1000$ generated
diagrams; the checkpoint with the smallest rare event probability error
was retained.

As a non-learning comparison, we implemented a self contained generator inspired by the generator for random PDs~\citep{Papamarkou2022randomPD}. We call it random persistence diagram generator (RPDG) style generator and it starts from a randomly selected observed diagram and chooses add, delete, and move operations with fixed probabilities. The original generic proposal settings produced diagrams that were poorly aligned with the coordinate scale, persistence distribution, and cardinality of the observed diagrams. We therefore adapted the proposal distributions to the present data while retaining the generator’s fixed, non-learning add--delete--move mechanism.
% To calibrate the add proposal to the observed data, pooled empirical PD points are divided according to whether their persistence \(d-b\) is at most or greater than \(\tau\). The high persistence pool is selected with its observed frequency among all pooled points, after which a point is sampled uniformly from the selected pool and perturbed by small Gaussian jitter. Delete operations remove a uniformly selected point, whereas move operations perturb the birth and persistence coordinates of a uniformly selected point; all proposed diagrams are projected back to the admissible wedge. For final evaluation, the generator produced \(500\) diagrams following \(60\) burn-in edits, with one diagram retained after every \(30\) subsequent edits.
\begin{table}[t]
\centering
\caption{Recovery of the target rare event probability in the synthetic
experiment.}
\label{tab:rare_event_results}
%\small
\setlength{\tabcolsep}{4pt}
\begin{tabular}{lcc}
\toprule
Method
& \shortstack{Rare event\\probability}
& \shortstack{Absolute\\error}\\
\midrule
Target
& $0.120$
& $0.000$\\
RPDG style  reference
& $0.012$
& $0.108$\\
RL policy
& $0.084$
& $0.036$\\
\bottomrule
\end{tabular}
\end{table}

\begin{figure}[t]
\centering
\includegraphics[width=0.75\linewidth]
{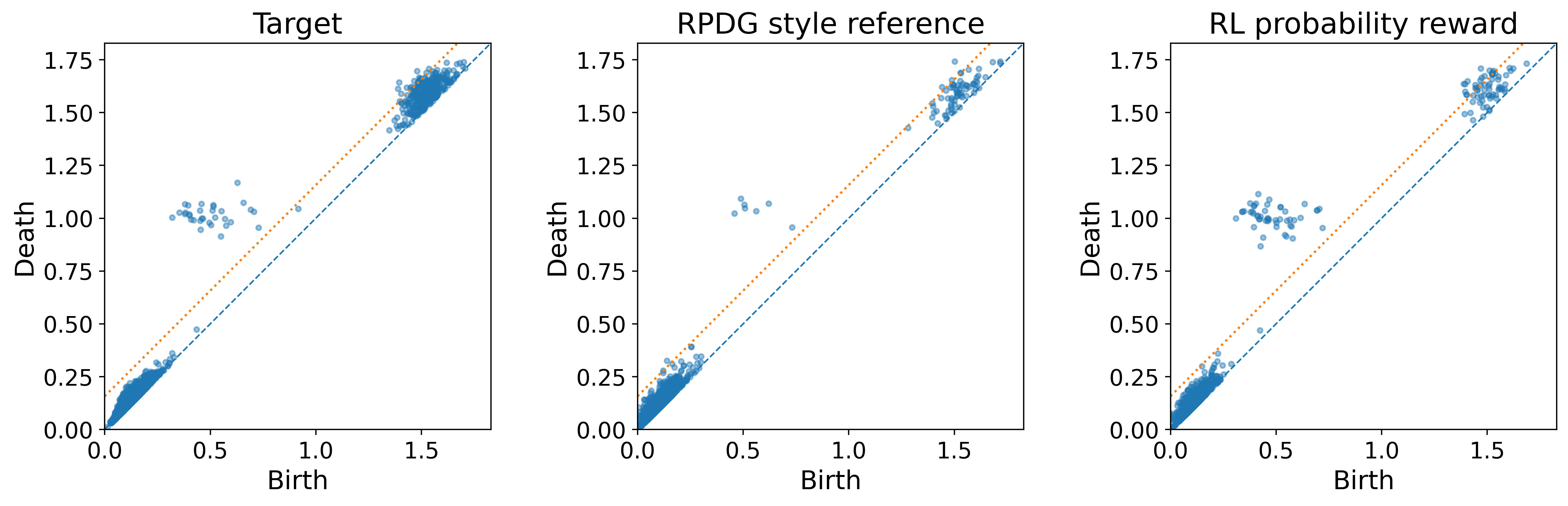}

\vspace{0.5em}

\includegraphics[width=0.75\linewidth]
{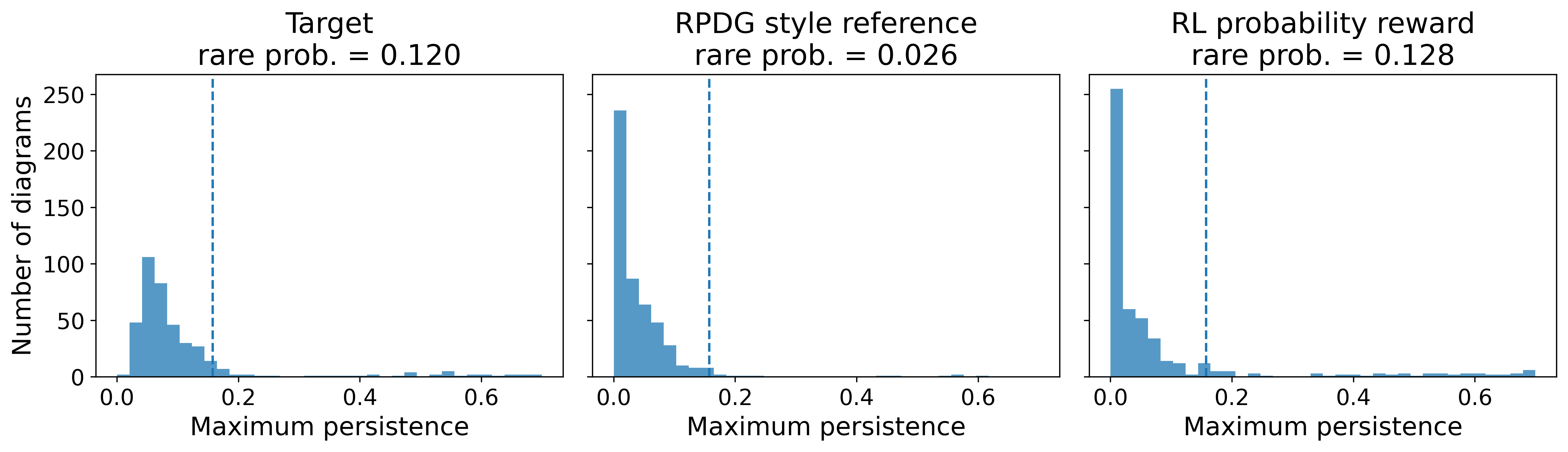}

\vspace{0.5em}

\includegraphics[width=0.5\linewidth]
{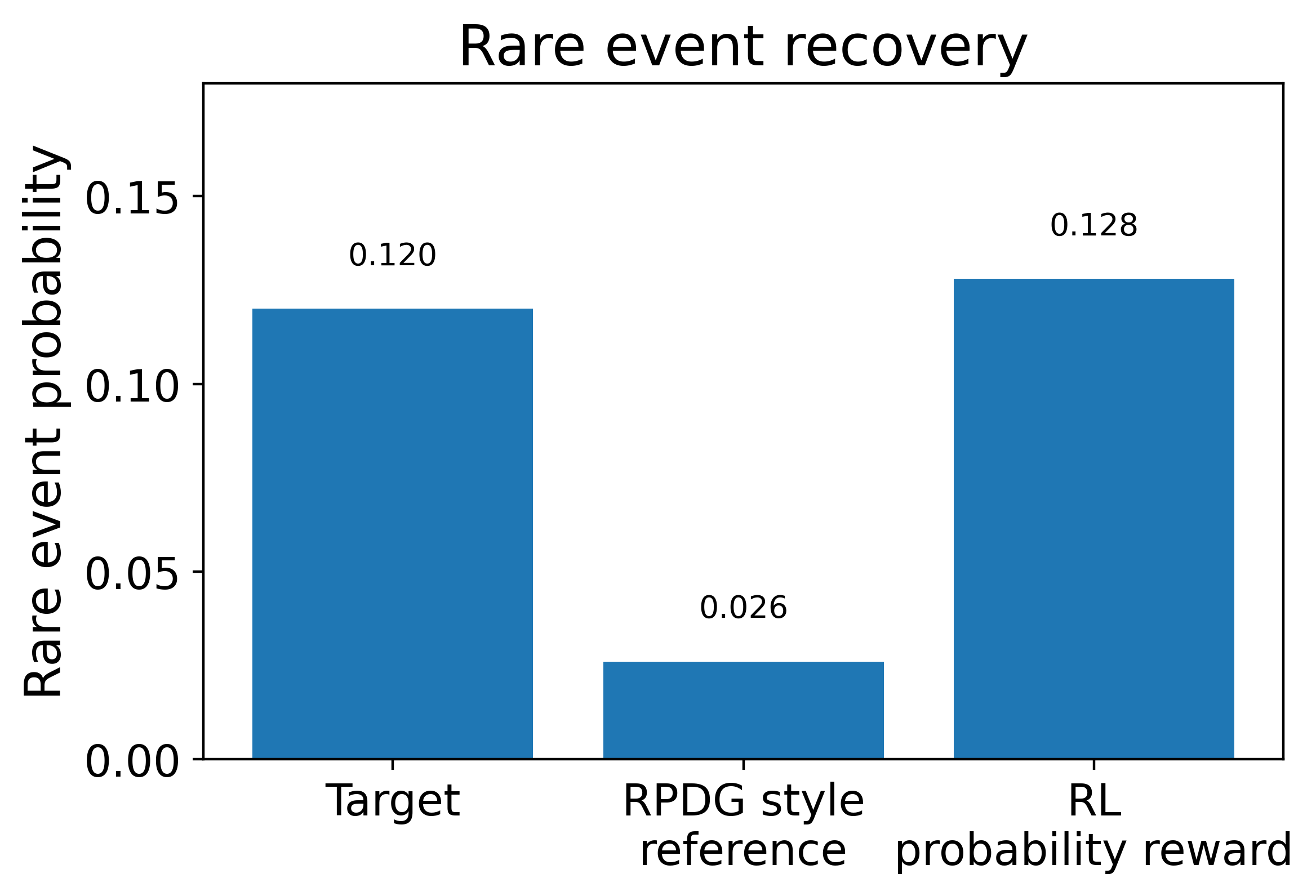}

\caption{Synthetic rare-event recovery results. The upper row compares
the pooled PD points, with the dotted line indicating
$d=b+\tau$. The middle row compares the distributions of maximum
persistence. The lower panel reports the target and generated
rare-event probabilities.}
\label{fig:rare_recovery_results}
\end{figure}

\subsubsection{Results and interpretation.}
The RPDG style reference produced a rare event probability of $0.012$,
giving an absolute error of $0.108$. The learned policy produced a
probability of $0.084$ and reduced the absolute error to $0.036$.
% Thus, the RL error was one third of the compared method's one, corresponding
% to a $66.7\%$ reduction.
Equivalently, the learned policy recovered
$70\%$ of the target rare event frequency, whereas the RPDG style 
reference recovered only $10\%$. Figure~\ref{fig:rare_recovery_results} shows how this improvement
appears in PD space. The pooled diagrams in the upper row locate the
generated points relative to the threshold line, while the middle row
aggregates this information at the diagram level through maximum
persistence. The RPDG style  reference concentrates most generated
diagrams below the threshold. In contrast, the learned policy increases
the upper tail mass, producing substantially more diagrams that satisfy
$R_\tau(D)=1$. The lower row reports the resulting improvement in the
target quantity directly.

This improvement illustrates how the pushforward objective and sequential control work together. By choosing
\(\phi_{\mathrm{syn}}(D)=R_\tau(D)\), the distributional feedback directly measures the rare topological event of interest, while the RL policy adapts its sequence of edit probabilities to reduce disagreement with the target distribution. Unlike fixed random edit dynamics, the learned process therefore moves substantially closer to the target rare event frequency. Other distributional characteristics can be incorporated through a richer task specific representation and discrepancy.

\subsection{Topology Preserving Compression of ABIDE PDs}
\label{sec:exp_abide}

We evaluate the proposed framework on resting-state functional magnetic resonance imaging (rs-fMRI) from the Autism Brain Imaging Data Exchange (ABIDE)~\citep{dimartino2014abide}. In neuroimaging studies, PDs may contain many topological features for each participant and may be compared repeatedly across large, heterogeneous cohorts. This experiment examines whether a diagnosis independent policy can learn a sequence of point deletions that reduces the cost of repeated PD comparisons while preserving individual diagram topology and the population geometry of the complete cohort. Because the analyzed cohort contains participants with autism spectrum disorder (ASD) and typically developing controls, we additionally examine whether population geometry is preserved separately within each diagnostic group.
% The evaluation has four components: comparison with a minimum persistence deletion heuristic, visualization of the compression trajectory, analysis of the learned deletion behavior, and external validation of neighborhood preservation. 
Figure~\ref{fig:abide_pipeline} summarizes the complete pipeline.

\begin{figure*}[t]
\centering
\includegraphics[width=\textwidth]
{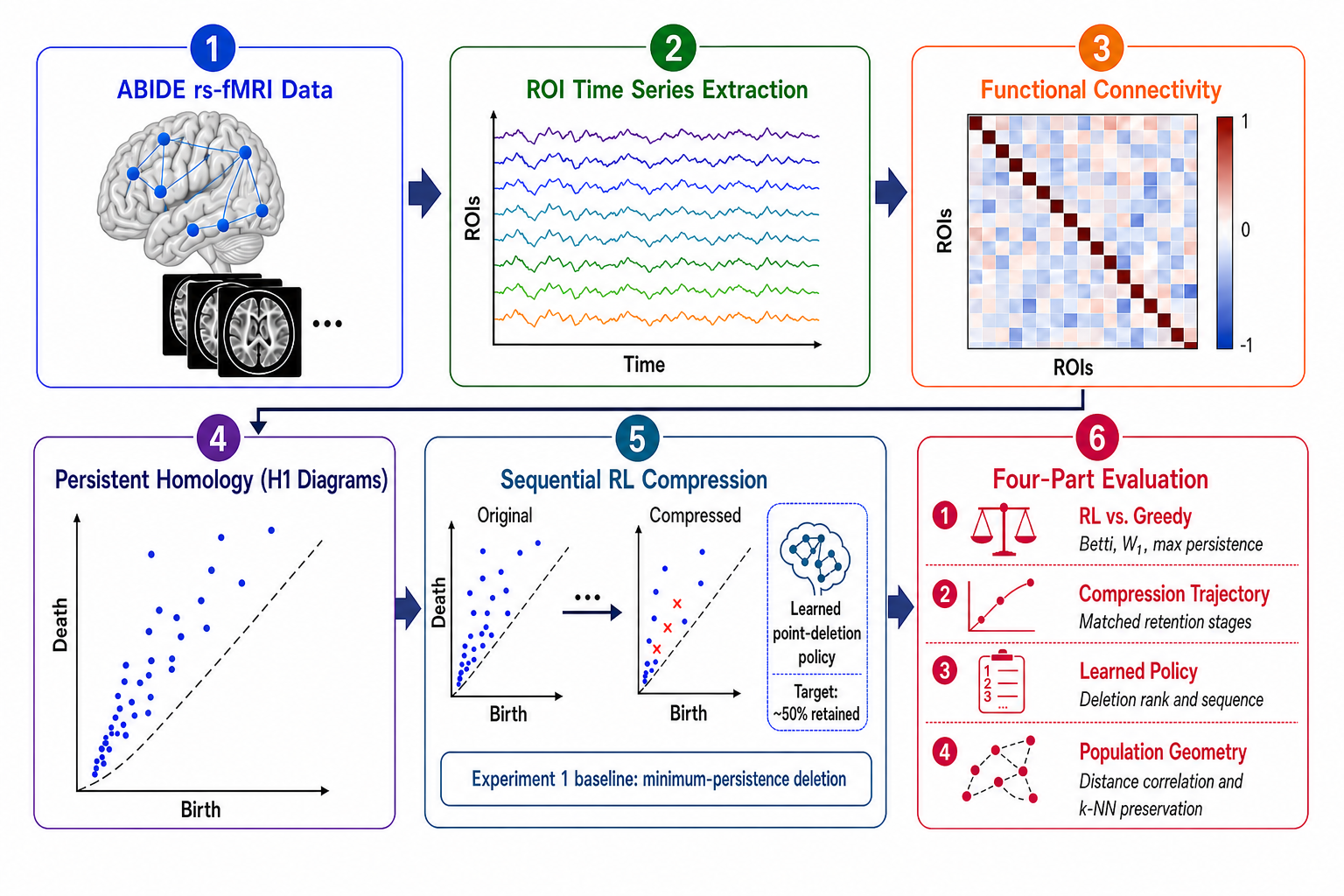}
\caption{ABIDE experimental pipeline.}
\label{fig:abide_pipeline}
\end{figure*}

\subsubsection{Data and PD construction.}
We use preprocessed region of interest (ROI) time series from 60 ABIDE participants, comprising 24 individuals with ASD and 36 controls, obtained with the C-PAC preprocessing pipeline and the Automated Anatomical Labeling atlas.
ROIs with zero temporal variance are removed separately for each
subject before computing the Pearson correlation matrix
$C=(C_{ij})$.  Connectivity is converted to the symmetric
dissimilarity
\[
\Delta_{ij}=1-\lvert C_{ij}\rvert,
\qquad
\Delta_{ii}=0.
\]
Thus, strong positive and strong negative correlations are both treated
as strong functional connectivity and enter the filtration early.  We
use $\Delta$ as the edge-weight matrix of a Rips-type clique
filtration: at threshold $\epsilon$, an edge $(i,j)$ is present when
$\Delta_{ij}\leq\epsilon$, and a higher dimensional simplex is present
when all of its edges are present.  We then compute $H_1$ points.
Across the 60 subjects, the resulting diagrams contain 3,715 points in
total, with a mean cardinality of 61.92.

\subsubsection{Sequential compression problem.}
For an original diagram
$D=\{(b_i,d_i)\}_{i=1}^{N}$, the state at step $t$ is the current
diagram $D_t$.  An action selects one of its points for deletion:
\[
\mathcal A(D_t)=\{1,\ldots,\lvert D_t\rvert\},
\qquad
D_{t+1}=D_t\setminus\{(b_{a_t},d_{a_t})\}.
\]
The episode begins at $D_0=D$ and ends when
$\lvert D_t\rvert=\lceil N/2\rceil$, where \(\lceil x\rceil\) denotes the smallest integer greater than or equal to \(x\).  Consequently, all methods are
compared at the same subject specific target cardinality.

\subsubsection{Finite horizon objective and reward.}
This experiment uses the finite horizon form of the general learning
framework rather than matching a stationary pushforward distribution.
Let
$
\operatorname{ret}_t
=
\frac{|D_t|}{|D|}
$
be the retained fraction.  We measure subject level distortion using
the $L^2$ distance $d_{\mathrm{Betti}}$ between Betti curves, the
 Wasserstein-1 distance $d_{\mathrm{pers}}$ between the
multisets of persistence values, and the maximum-persistence deviation
\begin{equation} \label{eq:d_max}
    d_{\max}(D,D_t)
=
\left|
p_{\max}(D)
-
p_{\max}(D_t)
\right|.
\end{equation}
The two persistence based terms play complementary roles: \(d_{\mathrm{pers}}\) measures changes across the complete persistence value distribution, whereas \(d_{\max}\) prevents the loss of the single most persistent feature, whose contribution to a distributional distance may be diluted in diagrams containing many points.
Unlike the synthetic experiment, the ABIDE experiment does not match a
generated distribution of PDs to an observed population distribution.
Instead, it uses a finite horizon, diagram level discrepancy that
compares each compressed diagram with its original diagram.
For a diagram $D=\{(b_i,d_i)\}_{i=1}^{N}$, and persistence values
$p_i=d_i-b_i$, let
\[
F_D(s)
=
\frac{1}{|D|}
\sum_{i=1}^n
\mathbf 1
\left\{
p_i\leq s
\right\}
\]
denote the empirical cumulative distribution function of its
persistence values. We define the persistence distribution discrepancy
between the original and current diagrams as
\begin{equation} \label{eq:d_pers}
d_{\mathrm{pers}}(D,D_t)
=
\int_{0}^{\infty}
\left|
F_D(s)-F_{D_t}(s)
\right|\,ds.
\end{equation}
This quantity is the Wasserstein-1 distance between
the empirical distributions of the persistence values. It is distinct
from the PD Wasserstein distance used later to compare
complete PDs in the cohort analysis.
We define the task specific
discrepancy by the weighted sum of the quantities in Equations \eqref{eq:d_max}, \eqref{eq:d_pers}, and $d_{\mathrm{Betti}}$ as
\[
\Dist_{\mathrm{task}}(D,D_t)
={}
\lambda_{\mathrm{Betti}}
d_{\mathrm{Betti}}(D,D_t)
+
\lambda_{\mathrm{pers}}
d_{\mathrm{pers}}(D,D_t)\\
+
\lambda_{\max}
d_{\max}(D,D_t).
\]
For a fixed original diagram $D$, the finite horizon compression
objective is
\[
J_D(D_t)
=
\Dist_{\mathrm{task}}(D,D_t)
-
\lambda_{\mathrm{comp}}(1-\operatorname{ret}_t).
\]
The first term
measures the topological distortion introduced by the deletions, while
the second term rewards reduction in diagram cardinality.
The step reward is the decrease in the task objective produced by the
selected deletion:
\[
r_t
=
J_D(D_t)-J_D(D_{t+1}).
\]
Thus, an action receives a larger reward when it removes a point while
causing only a small increase in topological discrepancy. 
% All four
% weights are set to one in the reported experiment. 
Because every
completed episode has the same terminal cardinality, the total amount
of compression is fixed; the policy must therefore distinguish
deletion sequences according to their effects on the three topological
distances.
% The Q-learning implementation optimizes the
% discounted cumulative reward using discount factor $\gamma=0.95$.

\subsubsection{Policy representation and training.}

The policy is conditioned on a seven dimensional feature
vector $\xi(D)$ of the
current state : retained fraction; mean, standard deviation, maximum,
and total persistence; and the mean and maximum of the Betti curve.
Each candidate point is described by five features: normalized birth,
death, and persistence, its current persistence rank, and its nearest
neighbor distance in the diagram.  A linear action-value model uses
the state features, candidate features, and their pairwise interactions
to score every admissible deletion.  Actions are selected by an
$\varepsilon$-greedy rule, and the action-value parameters are updated
by Q-learning with learning rate $0.01$, discount factor $0.95$, and
an exploration rate that decays from $1.0$ to a minimum of $0.05$.
The agent is trained for 2,000 episodes on one representative
46-point diagram and then applied, without subject-specific
retraining, to all 60 subjects.  This design tests whether the learned
sequential rule transfers across diagrams of different cardinalities.

\noindent \emph{Minimum persistence baseline.}
The comparison method is the existing greedy heuristic that deletes
the currently least-persistent point,
\[
a_t^{\mathrm{greedy}}
=
\arg\min_{i\in\mathcal A(D_t)}\,p_i.
\]
This is a natural local baseline because it treats the feature nearest
the diagonal as the safest deletion at each step.  Unlike the learned
policy, it does not evaluate the cumulative consequences of a sequence
of edits.

\noindent \emph{Experiment 1: Final individual level fidelity.}
Both methods start from the same subject diagram and stop at the same
target cardinality. Table~\ref{tab:abide_greedy} reports the final
distortions and retained fraction over all 60 subjects.
\begin{table}[t]
\centering
\caption{ABIDE compression at approximately $50\%$ retention. Values
are mean $\pm$ standard deviation over 60 subjects; lower distortion
is better.}
\label{tab:abide_greedy}
\small
\begin{tabular}{lccc}
\toprule
Method
& $d_{\mathrm{Betti}}$
& $d_{\mathrm{pers}}$
& Retained fraction\\
\midrule
RL policy
& $0.982\pm0.310$
& $0.0220\pm0.0051$
& $0.5045\pm0.0046$\\
Minimum persistence
& $0.988\pm0.303$
& $0.0222\pm0.0052$
& $0.5045\pm0.0046$\\
\bottomrule
\end{tabular}
\end{table}
At the same retained cardinality, the learned policy has modestly
lower mean Betti curve and persistence distribution distortion than
the greedy heuristic.  Both methods preserve the original maximum
persistence for every subject, so the corresponding maximum persistence
deviation is zero and is omitted from the table.  These results do not
establish a large separation in final fidelity; instead, they show
that the learned policy matches or slightly improves the two reported
individual level criteria while following a substantially different
deletion strategy.

\noindent \emph{Experiment 2: Sequential compression trajectories.}
Figure~\ref{fig:compression_trajectory} shows a representative diagram
at matched points along the two deletion trajectories.  The learned
policy does not simply reproduce the minimum persistence rule: it
selects deletions according to their predicted cumulative effect on
the multi-term objective.

\begin{figure}[t]
\centering
\includegraphics[width=0.75\textwidth]
{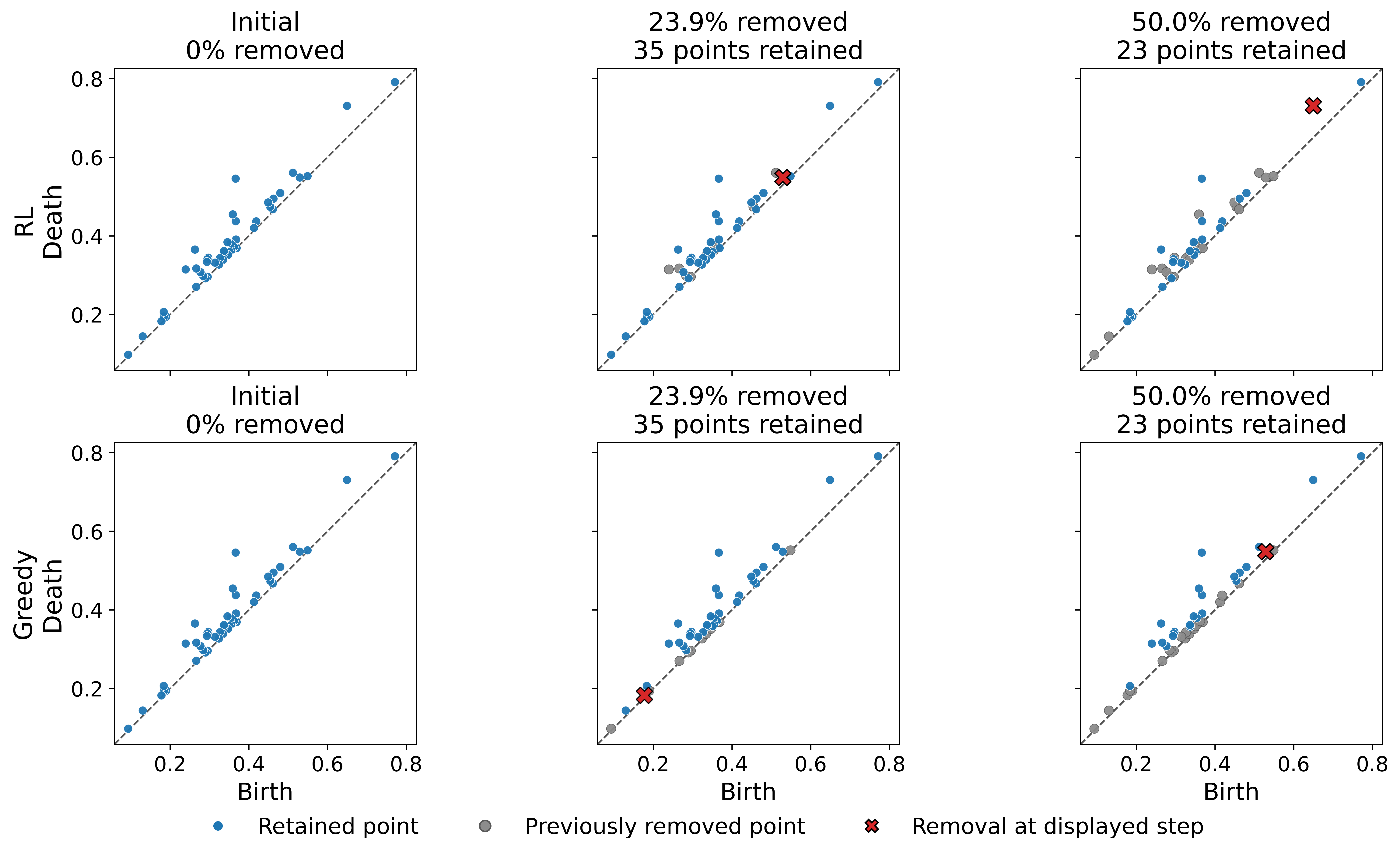}
\caption{Representative sequential compression trajectories for the
RL policy and the minimum persistence heuristic.
% Columns show the
% original diagram, intermediate states at matched retained fractions,
% and the final diagrams at approximately $50\%$ retention.
}
\label{fig:compression_trajectory}
\end{figure}

\noindent \emph{Experiment 3: learned deletion behavior.}
We examine one representative compression trajectory to determine
whether the learned policy simply repeats the minimum persistence rule.
At each step, we rank the points remaining in the diagram by
persistence, with rank one assigned to the point having the smallest
persistence. The greedy method always deletes the point with rank one.
The RL policy and the greedy method never select the same point at the
same step. Because the two methods can delete some of the same points
at different steps, their final sets of deleted points partially
overlap; the Jaccard overlap between these sets is $0.353$. The points
selected by the RL policy have a mean persistence rank of $17.65$ and a
median rank of $19$. Thus, the learned policy does not reproduce the
minimum persistence rule.

Table~\ref{tab:abide_policy_analysis} shows how the deletion choices
change during the trajectory. Compared with the greedy method, the RL
policy deletes a point with higher persistence in $100\%$ of the early
steps, $75\%$ of the middle steps, and $50\%$ of the late steps. The
policy therefore uses different deletion behavior at different stages
of compression rather than applying one fixed ranking rule throughout
the episode. These values describe the representative trajectory and
are not intended as cohort wide averages.
\begin{table}[t]
\centering
\caption{Deletion behavior in the representative RL trajectory.
``Higher persistence'' is the percentage of steps at which the RL
policy deletes a more persistent point than the greedy heuristic.}
\label{tab:abide_policy_analysis}
\small
\begin{tabular}{lccc}
\toprule
Stage & Steps & Higher persistence & Mean RL rank\\
\midrule
Early  & 7 & $100\%$ & 23.71\\
Middle & 8 & $75\%$  & 15.88\\
Late   & 8 & $50\%$  & 14.13\\
\bottomrule
\end{tabular}
\end{table}

\noindent \emph{Experiment 4: External validation of population geometry.}
\label{sec:population_geometry}
The policy is trained using diagram level fidelity terms only; it is
not given pairwise cohort distances, diagnostic labels, or participant
neighborhoods. We therefore use these quantities as an external test of
whether the compressed diagrams retain information needed for cohort
comparisons. Let $W=(W_{ij})$ and $\widetilde W=(\widetilde W_{ij})$
denote the pairwise PD Wasserstein distance matrices
before and after RL compression, respectively. Global distance
preservation is measured by the Pearson correlation between the upper-triangular entries of the original and compressed distance matrices:
\[
\rho_{\mathrm{dist}}
=
\operatorname{corr}
\left(
\operatorname{vec}_{\triangle}(W),
\operatorname{vec}_{\triangle}(\widetilde W)
\right).
\] For local geometry, let $\mathcal N_k(i)$ and
$\widetilde{\mathcal N}_k(i)$ be the sets of the $k$ nearest neighbors
of participant $i$ before and after compression. We define
\[
\operatorname{NP}_k
=
\frac{1}{n}
\sum_{i=1}^{n}
\frac{
\left|
\mathcal N_k(i)\cap\widetilde{\mathcal N}_k(i)
\right|
}{k}.
\]

Because the cohort contains both participants with ASD and controls, we compute these measures
for the complete cohort and separately within each diagnostic group.
% For the group specific analyses, the distance matrices and nearest
% neighbors are restricted to participants in the corresponding group.
This provides a robustness check of whether the overall preservation
result holds in both components of the cohort. Results are shown in Table \ref{tab:abide_population_geometry} and Figure \ref{fig:abide_population_geometry}.
\begin{table}[t]
\centering
\caption{Preservation of population geometry after approximately
$50\%$ of each diagram is retained. 
% For the complete cohort row,
% neighbors are selected from all 60 participants. For the diagnostic
% group rows, neighbors are selected only from within the corresponding
% group. Neighborhood preservation values are mean $\pm$ standard
% deviation across participants.
}
\label{tab:abide_population_geometry}
\small
\begin{tabular}{lccc}
\toprule
Analysis set
& $\rho_{\mathrm{dist}}$
& $\operatorname{NP}_5$
& $\operatorname{NP}_{10}$\\
\midrule
Complete cohort ($n=60$)
& $0.9950$
& $0.8267\pm0.1625$
& $0.8967\pm0.0736$\\
Autism group ($n=24$)
& $0.9971$
& $0.8833\pm0.1167$
& $0.9625\pm0.0576$\\
Control group ($n=36$)
& $0.9929$
& $0.8556\pm0.1403$
& $0.9444\pm0.0607$\\
\bottomrule
\end{tabular}
\end{table}
% For the complete cohort, the original and compressed distance matrices
% have correlation $\rho_{\mathrm{dist}}=0.9950$. The compressed diagrams
% also retain, on average, $82.7\%$ of the five nearest neighbors and
% $89.7\%$ of the ten nearest neighbors of each participant.
% Preservation remains high when the analysis is performed separately
% within each diagnostic group. The distance-matrix correlations are
% $0.9971$ for the autism group and $0.9929$ for the control group, while
% the corresponding neighborhood preservation values exceed $0.85$ for
% $k=5$ and $0.94$ for $k=10$.

Thus, after approximately half of every diagram has been removed, the
RL compressed diagrams closely preserve both global pairwise geometry
and local neighborhoods. The group specific results further show that
the overall preservation is not driven solely by one diagnostic group.
Because cohort distances, diagnostic labels, and neighborhoods do not
appear in the reward, these findings provide external evidence that the
learned edit policy retains meaningful cohort organization beyond the
diagram level criteria on which it was trained.

\begin{figure}[t]
\centering
\includegraphics[width=\textwidth]
{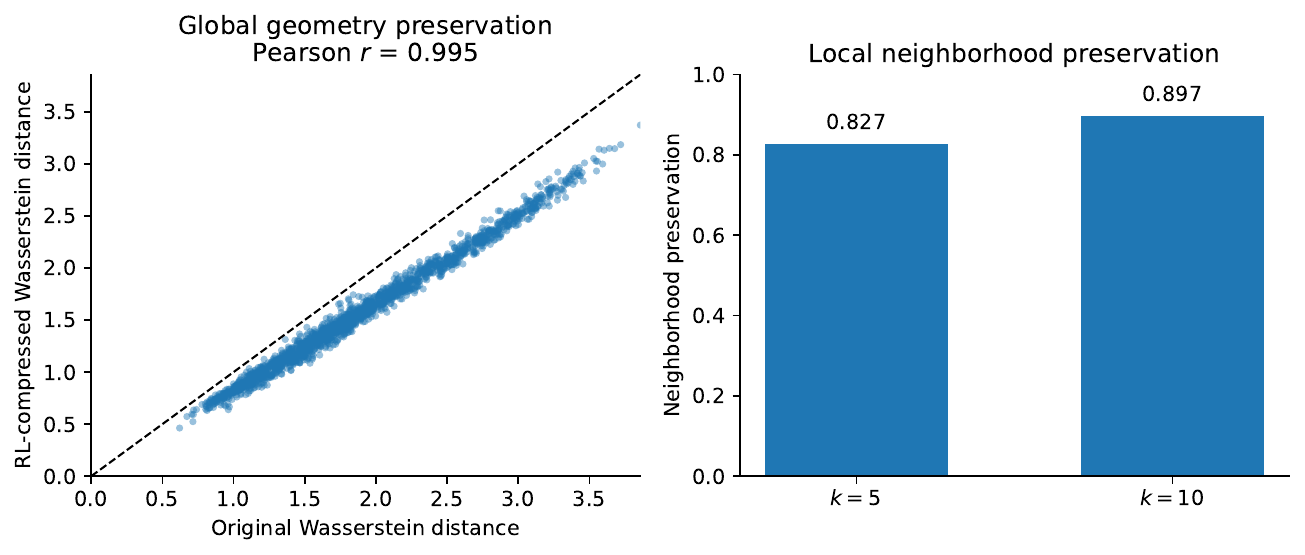}
\caption{Population geometry after RL compression. The first panel
compares the strict upper-triangular entries of the original and
compressed pairwise persistence-diagram Wasserstein distance matrices
for the complete cohort. The second panel reports complete-cohort
nearest-neighbor preservation for $k=5$ and $k=10$.}
\label{fig:abide_population_geometry}
\end{figure}
% Taken together, the four analyses show what the sequential formulation
% adds to a conventional deletion rule.  The learned policy attains
% comparable or modestly improved subject-level fidelity at matched
% compression, follows a measurably different and stage-dependent edit
% trajectory, transfers from one training diagram to the wider cohort,
% and preserves between-subject geometry that was not included in its
% reward.  The evidence therefore supports the more specific conclusion
% that reinforcement learning provides a viable non-greedy mechanism for
% topology-aware sequential compression, rather than the stronger claim
% that it uniformly dominates minimum-persistence deletion.

%%%%%%%%%%%%%%%%%%%%%%%%%%%%%%%%%%%%%%%%%%%%%%%%
\section{Discussion}
\label{sec:discussion}

This work introduces a process based perspective on PDs. Rather than treating a PD only as a static topological
summary, we model it as the state of a controlled stochastic system
that evolves through local, topology aware edits. RL provides a mechanism for selecting these edits according to
their cumulative consequences. The main innovation is therefore not
simply the application of a learning algorithm to features extracted
from PDs; it is the formulation of learnable stochastic dynamics acting
directly on PD space. This formulation supports both long run
distributional modeling and finite horizon tasks such as rare event
recovery and topology preserving compression, respectively.

The theoretical results provide a probabilistic foundation for this
perspective. Under the stated regularity conditions, the
policy controlled edit process is irreducible and aperiodic and
satisfies appropriate drift and minorization conditions, leading to a   
unique stationary probability law. The ergodic results justify using
finite policy generated trajectories to approximate expectations under
this law, while the consistency result connects empirical
distributional objectives to their population counterparts. Together,
these results explain why trajectories produced through successive
local edits can be used for probabilistic learning on a
variable cardinality and non-Euclidean state space. The theory applies
to the stationary formulation, while the same edit based framework can
also be used with task specific objectives.

The two experiments illustrate these complementary uses. In the
synthetic study, the objective emphasizes recovery of a rare
topological regime. The results
show that the policy controlled process can be directed toward
scientifically important but infrequent structure rather than merely
reproducing the most common diagrams. In the ABIDE study, the framework
is used for finite horizon compression. At approximately \(50\%\)
retention, the learned policy produces modestly lower average
Betti curve and persistence distribution distortion than the
minimum persistence heuristic, while both methods preserve maximum
persistence. These results do not establish uniform superiority over
greedy deletion, but they show that a learned sequential policy can
achieve comparable or slightly improved fidelity without being
restricted to a fixed pointwise ranking rule.

An interesting observation of the ABIDE study is that the learned
policy preserves structure that it is never explicitly trained to
preserve. Although the reward is defined entirely through
diagram level fidelity and compression, the compressed diagrams retain
both the global geometry and the large majority of local neighborhoods
of the original cohort. This preservation remains strong when the
ASD and control groups are examined separately, indicating that the
result is not driven by only one component of the cohort. The policy
achieves this without access to diagnostic labels, pairwise cohort
distances, or neighborhood information. 
% Moreover, the representative
% trajectory shows that the policy does not simply imitate
% minimum persistence deletion: it selects different points and changes
% its deletion behavior as compression proceeds.
These findings
show the principal advantage of the sequential formulation.
% : it can
% learn a diagnosis independent and context dependent deletion strategy that
% preserves organization beyond the quantities explicitly included in
% its reward.

Several methodological directions could extend this work and benefit the
broader research community. Future research may develop
richer edit mechanisms, more expressive policies for
variable cardinality diagrams, and adaptive objectives that balance
multiple topological and scientific requirements without relying on
fixed weights. Extending the theoretical analysis to broader policy
classes, continuous action parameters, and data dependent transition
mechanisms would further strengthen the connection between stochastic
process theory and RL on topological spaces.
Scalable training procedures will also be important for large
collections of diagrams and for settings involving several homological
dimensions. Finally, systematic benchmarks, repeated evaluations, and
uncertainty measures for generated or simplified diagrams would provide
a common basis for comparing methods and identifying when sequential
topological editing offers a genuine advantage.

This work establishes a foundation for a broader research direction in
which PDs are not only analyzed as final summaries, but are treated as
states of learnable stochastic processes. By combining controlled local
edits, probabilistic dynamics, and reinforcement learning, the proposed
framework connects sequential decision making with statistical modeling
directly on PD space. The theoretical results show that these dynamics
can possess well-defined long run behavior and support consistent
distributional learning, while the experiments demonstrate that they
can recover targeted rare structure and perform non-greedy compression
without destroying important diagram level or cohort level
organization. The central contribution is therefore a change in how PDs
can be modeled: from static outputs of a filtration to dynamic objects
that can be generated, transformed, and simplified according to a
scientific objective. This perspective opens a new route toward
adaptive and task aware methods for topological data analysis.
%%%%%%%%%%%%%%%%%%%%%%%%%%%%%%%%%%%%%%%%%%%%%%%%

\appendix

\section{Technical Proofs for Section \ref{sec:theory}}
\label{appendix:markov_proofs}

% This appendix provides the technical foundations for
% Theorem~\ref{thm:stationary}. We first construct a reference measure on
% persistence diagram (PD) space and then establish the probabilistic
% properties of the policy induced Markov chain (MC) required for existence and uniqueness of the stationary distribution.

This appendix gives sufficient conditions for irreducibility,
aperiodicity, geometric drift, and minorization of the policy-induced
Markov chain. Together, these conditions imply the stationary and
geometric-convergence results in Theorem~\ref{thm:stationary}.

\subsection{Reference Measure on PD Space}
\label{appendix:reference-measure}

To establish irreducibility of the policy induced MC, we
introduce a reference measure on PD space. 
For each $k\ge 0$, let $\mathcal D_K^k = \{D\in\mathcal D_K:\ |D|=k\}$
denote the subset of PDs with exactly $k$
points. Recall the definitions of $K$ and $\mathcal D_K$ from Section \ref{sec:state_action}. For $k\geq1$, let
$
K^k
=
\underbrace{K\times\cdots\times K}_{k\text{ times}}
\subset\mathbb R^{2k}$
denote the space of ordered $k$-tuples of admissible PD
points. Define the canonical quotient map
\[
\kappa_k:K^k\longrightarrow\mathcal D_K^k,
\qquad
\kappa_k(x_1,\ldots,x_k)
=
\{\!\{x_1,\ldots,x_k\}\!\},
\]
which sends an ordered tuple to the corresponding unordered PD. In particular, tuples that differ only by a permutation are
mapped to the same diagram.
Since $\kappa_k$ is Borel
measurable, for any Borel set $E\subset\mathcal D_K^k$ we define
\[
\psi_k(E)
=
\operatorname{Leb}_{2k}\!\bigl(\kappa_k^{-1}(E)\bigr),
\]
where $\operatorname{Leb}_{2k}$ denotes $2k$-dimensional Lebesgue
measure on $\mathbb R^{2k}$. Here, $\kappa_k^{-1}(E)$ denotes the preimage of the set $E$, not an
inverse function. The map $\kappa_k$ is generally not injective because
tuples that differ only by a permutation represent the same PD. Because
$\kappa_k^{-1}(E)\subseteq K^k$, the measure is evaluated only on the
admissible coordinate region $K^k$. For $k=0$, define
$
\psi_0=\delta_{\varnothing},
$
the Dirac measure concentrated at the empty diagram $\varnothing$.
Let $(w_k)_{k\geq0}$ be a strictly positive summable sequence; for
example, let $w_0=1$ and $w_k=2^{-k}$ for $k\geq1$. Define
\[
\psi(E)
=
\sum_{k=0}^{\infty}
w_k\,\psi_k(E\cap\mathcal D_K^k),
\qquad
E\in\mathcal B(\mathcal D_K).
\]
Then $\psi$ is a $\sigma$-finite measure on $\mathcal D_K$.
The reference measure is connected to the transition mechanism as
follows. If $E\in\mathcal B(\mathcal D_K)$ satisfies $\psi(E)>0$, then
there exists some $k\geq0$ such that
\[
\psi_k(E\cap\mathcal D_K^k)>0.
\]
Consequently,
\[
\operatorname{Leb}_{2k}
\left(
\kappa_k^{-1}(E\cap\mathcal D_K^k)
\right)
>0.
\]

Suppose that the policy assigns positive probability to the delete action
at every nonempty diagram and to the add action at every diagram.
Assume also that every existing point can be selected for deletion with
positive probability and that the conditional add-proposal density is
strictly positive on $K$. Starting from any finite diagram $D$, the
chain can then reach the empty diagram through a finite sequence of
delete operations with positive probability. From the empty diagram, it
can reach $E\cap\mathcal D_K^k$ through $k$ consecutive add operations
with positive probability. Therefore, every set of positive $\psi$
measure is accessible from every initial diagram, which establishes
$\psi$-irreducibility. 

Furthermore, for $n\geq0$, let
\begin{equation} \label{eq:n-step-prob}
P_\theta^n(E\mid D)
=
\Pr(D_n\in E\mid D_0=D)
\end{equation}
denote the $n$-step transition probability of the policy induced MC. By convention,
$
P_\theta^0(E\mid D)=\mathbf 1_E(D)
$ \citep{meyn2012markov}.

\subsection{Propositions supporting Theorem~\ref{thm:stationary}}
\label{appendix:props}

In this section we establish fundamental probabilistic properties required for Theorem~\ref{thm:stationary}, including
irreducibility, aperiodicity, drift, and minorization of the policy induced Markov
chain on $\mathcal D_K$.
\begin{proposition}
\label{prop:irreducible_app}
Assume that the policy assigns positive probability to every admissible
add and delete action, every point of a nonempty diagram can be selected
for deletion with positive probability, and the conditional add-proposal
density is strictly positive on $K$. Then the policy induced Markov
chain on $\mathcal D_K$ is $\psi$-irreducible.
\end{proposition}

\begin{proof}
Fix an initial diagram $D\in\mathcal D_K$ and a measurable set
$E\in\mathcal B(\mathcal D_K)$ satisfying $\psi(E)>0$. We show that
there exists an integer $n\geq1$ such that
$
P_\theta^n(E\mid D)>0,
$
where $P_\theta^n(E\mid D)$ is defined in Equation \eqref{eq:n-step-prob}.
Let $m=|D|$ denote the number of points in the initial diagram. By the
definition of $\psi$, there exists some $k\geq0$ such that, with
$
E_k
=
E\cap\mathcal D_K^k,
$
we have
$
\psi_k(E_k)>0.
$

First suppose that $k\geq1$. By the assumptions on the policy and the
delete mechanism, the chain can remove all $m$ points of $D$ and reach
the empty diagram in $m$ steps with positive probability:
\[
P_\theta^m(\{\varnothing\}\mid D)>0,
\]
where the expression equals one when $m=0$.
By the construction of $\psi_k$,
\[
\operatorname{Leb}_{2k}
\left(
\kappa_k^{-1}(E_k)
\right)
>0.
\]
Starting from $\varnothing$, consider $k$ consecutive add operations.
Because the policy selects the add action with positive probability and
the conditional add proposal density is strictly positive on $K$, every
subset of $K^k$ with positive $2k$-dimensional Lebesgue measure is
generated with positive probability. Therefore,
\[
P_\theta^k(E_k\mid\varnothing)>0.
\]
By the Chapman--Kolmogorov identity,
\[
P_\theta^{m+k}(E\mid D)
=
\int_{\mathcal D_K}
P_\theta^k(E\mid D')\,
P_\theta^m(dD'\mid D).
\]
Restricting the integral to the intermediate state
$\varnothing$ and using $E_k\subseteq E$, we obtain
\[
P_\theta^{m+k}(E\mid D)
\geq
P_\theta^m(\{\varnothing\}\mid D)\,
P_\theta^k(E\mid\varnothing)
\geq
P_\theta^m(\{\varnothing\}\mid D)\,
P_\theta^k(E_k\mid\varnothing)
>0.
\]
It remains to consider $k=0$. In this case,
$\psi_0(E_0)>0$ implies that $\varnothing\in E$. If $m\geq1$, repeated
delete operations give
\[
P_\theta^m(E\mid D)
\geq
P_\theta^m(\{\varnothing\}\mid D)
>0.
\]
If $m=0$, the chain starts at $\varnothing$. It can perform one add
operation followed by one delete operation, both with positive
probability, and return to $\varnothing$. Hence,
$
P_\theta^2(E\mid\varnothing)>0.
$
Thus, for every $D\in\mathcal D_K$ and every
$E\in\mathcal B(\mathcal D_K)$ satisfying $\psi(E)>0$, there exists
$n\geq1$ such that
\[
P_\theta^n(E\mid D)>0.
\]
Therefore, the policy induced MC is $\psi$-irreducible.
\end{proof}

\begin{proposition}
\label{prop:aperiodic_app}
Under the assumptions of Proposition~\ref{prop:irreducible_app},
suppose additionally that the policy selects the \emph{move} action
with positive probability at every singleton diagram and that the move
proposal has positive probability of remaining in \(K\). Then the
policy-induced Markov chain is aperiodic.
\end{proposition}

\begin{proof}
By Proposition~\ref{prop:irreducible_app}, the policy induced Markov
chain is $\psi$-irreducible. Recall that
$\psi_0=\delta_{\varnothing}$ and $w_0>0$, so
$
\psi(\{\varnothing\})=w_0>0.
$
Starting from the empty diagram, the chain can return to
$\varnothing$ in two steps by first adding a point and then deleting
that point. Because the policy assigns positive probability to both
actions, the add proposal assigns positive probability to admissible
points in $K$, and the unique point can subsequently be selected for
deletion,
\[
P_\theta^2(\{\varnothing\}\mid\varnothing)>0.
\]
The chain can also return to $\varnothing$ in three steps by adding a
point, moving it within $K$, and then deleting it. The move proposal has
positive probability of producing an admissible perturbation, and the
policy assigns positive probability to each of these actions. Therefore,
\[
P_\theta^3(\{\varnothing\}\mid\varnothing)>0.
\]
Define the set of possible return times to the empty diagram by
\[
\operatorname{per}_{\varnothing}
=
\left\{
n\geq1:
P_\theta^n(\{\varnothing\}\mid\varnothing)>0
\right\}.
\]
The preceding arguments show that both $2$ and $3$ belong to this set
of possible return times. Since $\gcd(2,3)=1$, we have
$
\operatorname{per}(\varnothing)=1.
$
Because the chain is $\psi$-irreducible and
$\psi(\{\varnothing\})>0$, the policy induced MC is
aperiodic.
\end{proof}
% Full proof of aperiodicity goes here.

Irreducibility and aperiodicity describe the accessibility and cyclic
behavior of the chain, but they do not control its behavior far from
the target diagrams. To formulate the Lyapunov drift condition needed
for the subsequent stability analysis, we introduce the one-step
transition operator
\[
(P_\theta f)(D)
:=
\mathbb E_\theta
\left[
f(D_{t+1})\mid D_t=D
\right]
\]
for any nonnegative measurable function
$f:\mathcal D_K\to\mathbb R$.

\begin{proposition}[Geometric Foster--Lyapunov drift]
\label{prop:drift_app}
Let \(\{D_i^*\}_{i=1}^{n}\subset\mathcal D_K\) be fixed target PDs and
define
\[
V(D)
=
1+
\frac{1}{n}
\sum_{i=1}^{n}
W_2(D,D_i^*)^2.
\]
% For a measurable function \(f:\mathcal D_K\to[0,\infty)\), write
% \[
% (P_\theta f)(D)
% =
% \mathbb E_\theta
% \left[
% f(D_{t+1})\mid D_t=D
% \right].
% \]
Suppose that there exist a measurable set
\(\mathcal G_V\subset\mathcal D_K\) and a constant
\(\eta_V\in(0,1)\) such that
\[
\mathbb E_\theta
\left[
V(D_{t+1})-V(D_t)\mid D_t=D
\right]
\leq
-\eta_V V(D),
\qquad D\notin\mathcal G_V.
\]
Assume also that
\[
c_V
=
\sup_{D\in\mathcal G_V}
\left[
(P_\theta V)(D)
-
(1-\eta_V)V(D)
\right]_+
<\infty.
\]
Then
\[
(P_\theta V)(D)
\leq
\lambda_V V(D)
+
c_V\mathbf 1_{\mathcal G_V}(D),
\qquad D\in\mathcal D_K,
\]
where \(\lambda_V=1-\eta_V\in(0,1)\).
\end{proposition}

\begin{proof}
Because conditioning on \(D_t=D\) fixes the current diagram,
\[
\mathbb E_\theta
\left[
V(D_{t+1})-V(D_t)\mid D_t=D
\right]
=
(P_\theta V)(D)-V(D).
\]
For \(D\notin\mathcal G_V\), the assumed drift inequality gives
\[
(P_\theta V)(D)-V(D)
\leq
-\eta_V V(D),
\]
and hence
\[
(P_\theta V)(D)
\leq
(1-\eta_V)V(D)
=
\lambda_VV(D).
\]
For \(D\in\mathcal G_V\), the definition of \(c_V\) gives
\[
(P_\theta V)(D)
\leq
\lambda_VV(D)+c_V.
\]
Combining the two cases yields
\[
(P_\theta V)(D)
\leq
\lambda_VV(D)
+
c_V\mathbf 1_{\mathcal G_V}(D).
\]
\end{proof}

% Full proof of drift condition goes here.

The final condition is minorization. A measurable set
\(\mathcal H\subseteq\mathcal D_K\) is called an \(m\)-step small set
if there exist an integer \(m\geq1\), a constant
\(\varepsilon_{\mathrm{min}}>0\), and a probability measure \(\nu\) on
\(\mathcal D_K\) such that
\[
P_\theta^m(E\mid D)
\geq
\varepsilon_{\mathrm{min}}\nu(E)
\]
for every \(D\in\mathcal H\) and every
\(E\in\mathcal B(\mathcal D_K)\). This minorization condition means
that the \(m\)-step transition distributions from all diagrams in
\(\mathcal H\) share a common probabilistic component. In the
geometric drift theorem, the condition is imposed on the same set
\(\mathcal G_V\) that appears in the Foster--Lyapunov inequality.

We establish this condition through a uniform reset mechanism. From
every \(D\in\mathcal G_V\), the chain reaches the empty diagram
\(\varnothing\) within a fixed number \(m_0\) of steps with a uniformly
positive probability. An additional add operation then generates a
singleton PD from a proposal density bounded below on
\(K_{\mathrm{add}}\). This produces a common lower bound for
\(P_\theta^{m_0+1}(E\mid D)\) over all
\(D\in\mathcal G_V\); other possible transitions contribute
nonnegative probability and therefore do not affect this lower bound.

\begin{proposition}[Minorization on the drift set]
\label{prop:minorization_app}
Assume that there exist an integer \(m_0\geq1\) and a constant
\(\alpha_0>0\) such that
\[
P_\theta^{m_0}
\left(
\{\varnothing\}\mid D
\right)
\geq
\alpha_0,
\qquad D\in\mathcal G_V.
\]
Suppose also that, from \(\varnothing\), the policy selects the
\emph{add} action with probability at least \(\alpha_{\mathrm{add}}>0\)
and that the add-proposal density satisfies
\[
q_{\mathrm{add}}(x\mid\varnothing)
\geq
q_{\min}>0,
\qquad x\in K_{\mathrm{add}},
\]
for some Borel set \(K_{\mathrm{add}}\subset K\) with
\(\operatorname{Leb}_2(K_{\mathrm{add}})>0\).

Then \(\mathcal G_V\) is an
\((m_0+1)\)-step small set. In particular, there exist
\(\varepsilon_{\min}>0\) and a probability measure \(\nu\) such that
\[
P_\theta^{m_0+1}(E\mid D)
\geq
\varepsilon_{\min}\nu(E),
\qquad
D\in\mathcal G_V,\quad
E\in\mathcal B(\mathcal D_K).
\]
\end{proposition}

\begin{proof}
Define
\[
\nu(E)
=
\frac{
\operatorname{Leb}_2
\left(
\{x\in K_{\mathrm{add}}:\{\!\{x\}\!\}\in E\}
\right)
}{
\operatorname{Leb}_2(K_{\mathrm{add}})
}.
\]
Starting from \(\varnothing\), the add-proposal assumptions give
\[
P_\theta(E\mid\varnothing)
\geq
\alpha_{\mathrm{add}}q_{\min}
\operatorname{Leb}_2(K_{\mathrm{add}})\nu(E).
\]
For \(D\in\mathcal G_V\), the Chapman--Kolmogorov identity and the
uniform reset assumption therefore imply
\[
\begin{split}
P_\theta^{m_0+1}(E\mid D)
&\geq
P_\theta^{m_0}
\left(
\{\varnothing\}\mid D
\right)
P_\theta(E\mid\varnothing)\\
&\geq
\alpha_0\alpha_{\mathrm{add}}q_{\min}
\operatorname{Leb}_2(K_{\mathrm{add}})\nu(E).
\end{split}
\]
Thus the minorization condition holds with
\[
\varepsilon_{\min}
=
\alpha_0\alpha_{\mathrm{add}}q_{\min}
\operatorname{Leb}_2(K_{\mathrm{add}})
>0.
\]
Hence \(\mathcal G_V\) is an \((m_0+1)\)-step small set.
\end{proof}

\bibliographystyle{plain}
\bibliography{reference}

\newcommand{\noop}[1]{}
\begin{thebibliography}{10}

\bibitem{adams2017persistence}
Henry Adams, Andrew Emerson, Michael Kirby, Rachel Neville, Chris Peterson,
  Patrick Shipman, Sofya Chepushtanova, Eric Hanson, Francesco Motta, and Lori
  Ziegelmeier.
\newblock Persistence images: A stable vector representation of persistent
  homology.
\newblock {\em Journal of Machine Learning Research}, 18(8):1--35, 2017.

\bibitem{Adler2019}
R.J. Adler and S.~Agami.
\newblock Modelling persistence diagrams with planar point processes, and
  revealing topology with bagplots.
\newblock {\em J Appl. and Comput. Topology}, 3:139–183, 2019.

\bibitem{Amezquita2023}
E.~J. Amezquita, F.~Nasrin, K.~M. Storey, and M.~Yoshizawa.
\newblock Genomics data analysis via spectral shape and topology.
\newblock {\em PLOS ONE}, 18(4):1--19, 2023.

\bibitem{Bernstein2020}
A.~Bernstein, E.~Burnaev, M.~Sharaev, E.~Kondrateva, and O.~Kachan.
\newblock Topological data analysis in computer vision.
\newblock In {\em Twelfth International Conference on Machine Vision (ICMV
  2019)}, volume 11433, page 114332H. International Society for Optics and
  Photonics, 2020.

\bibitem{bubenik2015landscapes}
Peter Bubenik.
\newblock Statistical topological data analysis using persistence landscapes.
\newblock {\em Journal of Machine Learning Research}, 16(3):77--102, 2015.

\bibitem{Carrire2019PersLayAN}
Mathieu Carriere, Frederic Chazal, Yuichi Ike, Theo Lacombe, Martin Royer, and
  Yuhei Umeda.
\newblock Perslay: A neural network layer for persistence diagrams and new
  graph topological signatures.
\newblock In {\em Proceedings of the Twenty Third International Conference on
  Artificial Intelligence and Statistics}, volume 108 of {\em Proceedings of
  Machine Learning Research}, pages 2786--2796. PMLR, 26--28 Aug 2020.

\bibitem{carriere2017sliced}
Mathieu Carri{\`e}re, Marco Cuturi, and Steve Oudot.
\newblock Sliced {W}asserstein kernel for persistence diagrams.
\newblock In {\em Proceedings of the 34th International Conference on Machine
  Learning (ICML)}, volume~70 of {\em Proceedings of Machine Learning
  Research}, pages 664--673, 2017.

\bibitem{chazal2015stochastic}
Fr\'{e}d\'{e}ric Chazal, Brittany~Terese Fasy, Fabrizio Lecci, Alessandro
  Rinaldo, and Larry Wasserman.
\newblock Stochastic convergence of persistence landscapes and silhouettes.
\newblock In {\em Proceedings of the Thirtieth Annual Symposium on
  Computational Geometry}, SOCG'14, page 474–483, New York, NY, USA, 2014.
  Association for Computing Machinery.

\bibitem{cohen2007stability}
David Cohen-Steiner, Herbert Edelsbrunner, and John Harer.
\newblock Stability of persistence diagrams.
\newblock {\em Discrete \& Computational Geometry}, 37(1):103--120, 2007.

\bibitem{Deshmukh2023}
V.~Deshmukh, S.~Baskar, T.~E. Berger, E.~Bradley, and J.~D. Meiss.
\newblock Comparing feature sets and machine-learning models for prediction of
  solar flares: Topology, physics, and model complexity.
\newblock {\em Astronomy \& Astrophysics}, 674:A159, 2023.

\bibitem{dimartino2014abide}
Adriana Di~Martino, Chao-Gan Yan, Qian Li, Erin Denio, F.~Xavier Castellanos,
  Kaat Alaerts, Jeffrey~S. Anderson, Michal Assaf, Susan~Y. Bookheimer, Mirella
  Dapretto, et~al.
\newblock The autism brain imaging data exchange: Towards a large-scale
  evaluation of the intrinsic brain architecture in autism.
\newblock {\em Molecular Psychiatry}, 19(6):659--667, 2014.

\bibitem{edelsbrunner2008persistent}
Herbert Edelsbrunner and John Harer.
\newblock Persistent homology---a survey.
\newblock {\em Contemporary Mathematics}, 453:257--282, 2008.

\bibitem{Gabella2021}
M.~Gabella.
\newblock Topology of learning in feedforward neural networks.
\newblock {\em IEEE Transactions on Neural Networks and Learning Systems},
  32(8):3588--3592, 2021.

\bibitem{ghrist2008barcodes}
Robert Ghrist.
\newblock Barcodes: The persistent topology of data.
\newblock {\em Bulletin of the American Mathematical Society}, 45(1):61--75,
  2008.

\bibitem{gretton2012kernel}
Arthur Gretton, Karsten~M Borgwardt, Malte~J Rasch, Bernhard Schölkopf, and
  Alexander Smola.
\newblock A kernel two-sample test.
\newblock In {\em Journal of Machine Learning Research}, volume~13, pages
  723--773, 2012.

\bibitem{haarnoja2018soft}
Tuomas Haarnoja, Aurick Zhou, Pieter Abbeel, and Sergey Levine.
\newblock Soft actor-critic: Off-policy maximum entropy deep reinforcement
  learning with a stochastic actor.
\newblock In {\em International Conference on Machine Learning}, pages
  1861--1870. PMLR, 2018.

\bibitem{Heydenreich2021}
S.~Heydenreich, B.~Br{\"u}ck, and J.~Harnois-D{\'e}raps.
\newblock Persistent homology in cosmic shear: Constraining parameters with
  topological data analysis.
\newblock {\em Astronomy \& Astrophysics}, 648:A74, 2021.

\bibitem{hofer2019learning}
Christoph Hofer, Roland Kwitt, Marc Niethammer, and Andreas Uhl.
\newblock Learning representations of persistence barcodes.
\newblock {\em Journal of Machine Learning Research}, 20(126):1--45, 2019.

\bibitem{Ichinomiya2020}
T.~Ichinomiya, I.~Obayashi, and Y.~Hiraoka.
\newblock Protein-folding analysis using features obtained by persistent
  homology.
\newblock {\em Biophysical Journal}, 118(12):2926--2937, 2020.

\bibitem{kusano2016persistence}
Genki Kusano, Kenji Fukumizu, and Yasuaki Hiraoka.
\newblock Persistence weighted gaussian kernel for topological data analysis.
\newblock In {\em Proceedings of the 33rd International Conference on Machine
  Learning (ICML)}, volume~48 of {\em Proceedings of Machine Learning
  Research}, pages 2004--2013, 2016.

\bibitem{kwitt2015statistical}
Roland Kwitt, Stefan Huber, Marc Niethammer, Weili Lin, and Ulrich Bauer.
\newblock Statistical topological data analysis - a kernel perspective.
\newblock In C.~Cortes, N.~Lawrence, D.~Lee, M.~Sugiyama, and R.~Garnett,
  editors, {\em Advances in Neural Information Processing Systems}, volume~28.
  Curran Associates, Inc., 2015.

\bibitem{Liu2019}
S.~Liu et~al.
\newblock Scalable topological data analysis and visualization for evaluating
  data-driven models in scientific applications.
\newblock {\em IEEE Transactions on Visualization and Computer Graphics},
  26(1):291--300, 2019.

\bibitem{Maroulas2021}
V.~Maroulas, C.~Micucci, and F.~Nasrin.
\newblock Topological learning for classifying the structure of biological
  networks.
\newblock {\em Bayesian Analysis}, 2021.

\bibitem{maroulas2019nonparametric}
Vasileios Maroulas, Joshua~L. Mike, and Christopher Oballe.
\newblock Nonparametric estimation of probability density functions of random
  persistence diagrams.
\newblock {\em Journal of Machine Learning Research}, 20(151):1--49, 2019.

\bibitem{maroulas2019bayesian}
Vasileios Maroulas, Farzana Nasrin, and Christopher Oballe.
\newblock A {B}ayesian framework for persistent homology.
\newblock {\em SIAM Journal on Mathematics of Data Science}, 1(1):48--74, 2019.

\bibitem{meyn2012markov}
Sean Meyn and Richard~L Tweedie.
\newblock {\em Markov Chains and Stochastic Stability}.
\newblock Springer, 2012.

\bibitem{mileyko2011probability}
Yuriy Mileyko, Sayan Mukherjee, and John Harer.
\newblock Probability measures on the space of persistence diagrams.
\newblock {\em Inverse Problems}, 27(12):124007, 2011.

\bibitem{moller2003statistical}
Jesper Møller and Rasmus~Plenge Waagepetersen.
\newblock {\em Statistical Inference and Simulation for Spatial Point
  Processes}.
\newblock CRC Press, 2003.

\bibitem{Nasrin_2024}
Farzana Nasrin, Theodore Papamarkou, Austin Lawson, Na~Gong, Orlando Rios, and
  Vasileios Maroulas.
\newblock Bayesian random persistence diagram generation: An application to
  material microstructure analysis.
\newblock {\em Foundations of Data Science}, 6(3):361--378, 2024.

\bibitem{ng1999policy}
Andrew~Y. Ng, Daishi Harada, and Stuart Russell.
\newblock Policy invariance under reward transformations: Theory and
  application to reward shaping.
\newblock In {\em Proceedings of the Sixteenth International Conference on
  Machine Learning}, pages 278--287. Morgan Kaufmann, 1999.

\bibitem{oudot2015persistence}
Steve~Y. Oudot.
\newblock {\em Persistence theory: From quiver representations to data
  analysis}, volume 209 of {\em Mathematical Surveys and Monographs}.
\newblock American Mathematical Society, 2015.

\bibitem{Papamarkou2022randomPD}
Theodore Papamarkou, Farzana Nasrin, Austin Lawson, Na~Gong, Orlando Rios, and
  Vasileios Maroulas.
\newblock Random persistence diagram generator.
\newblock {\em Statistics and Computing}, 32(5):88, 2022.

\bibitem{reininghaus2015stable}
Jan Reininghaus, Stefan Huber, Ulrich Bauer, and Roland Kwitt.
\newblock A stable multi-scale kernel for topological machine learning.
\newblock In {\em Proceedings of the IEEE Conference on Computer Vision and
  Pattern Recognition (CVPR)}, pages 4741--4748, 2015.

\bibitem{sheehy2021sketching}
Donald~R. Sheehy and Siddharth Sheth.
\newblock Sketching persistence diagrams.
\newblock In {\em 37th International Symposium on Computational Geometry (SoCG
  2021)}, volume 189 of {\em Leibniz International Proceedings in Informatics
  (LIPIcs)}, pages 57:1--57:15. Schloss Dagstuhl--Leibniz-Zentrum f{\"u}r
  Informatik, 2021.

\bibitem{Sun2021}
H.~Sun, W.~Manchester~IV, and Y.~Chen.
\newblock Improved and interpretable solar flare predictions with spatial and
  topological features of the polarity inversion line masked magnetograms.
\newblock {\em Space Weather}, 19(12):e2021SW002837, 2021.

\bibitem{sutton2018reinforcement}
Richard~S Sutton and Andrew~G Barto.
\newblock {\em Reinforcement Learning: An Introduction}.
\newblock MIT Press, 2nd edition, 2018.

\bibitem{Townsend2020}
J.~Townsend, C.~P. Micucci, J.~H. Hymel, V.~Maroulas, and K.~D. Vogiatzis.
\newblock Representation of molecular structures with persistent homology for
  machine learning applications in chemistry.
\newblock {\em Nature Communications}, 11(1):1--9, 2020.

\bibitem{turner2014frechet}
Katharine Turner, Yuriy Mileyko, Sayan Mukherjee, and John Harer.
\newblock Fr{\'e}chet means for distributions of persistence diagrams.
\newblock {\em Discrete \& Computational Geometry}, 52(1):44--70, 2014.

\bibitem{Tymochko2020}
S.~Tymochko, E.~Munch, J.~Dunion, K.~Corbosiero, and R.~Torn.
\newblock Using persistent homology to quantify a diurnal cycle in hurricanes.
\newblock {\em Pattern Recognition Letters}, 133:137--143, 2020.

\bibitem{villani2009optimal}
C{\'e}dric Villani.
\newblock {\em Optimal Transport: Old and New}, volume 338 of {\em Grundlehren
  der mathematischen Wissenschaften}.
\newblock Springer, Berlin, Heidelberg, 2009.

\end{thebibliography}

\end{document}